\documentclass{article} 
\usepackage{fullpage}
\usepackage{amsmath,amssymb,amsfonts,amsthm}
\usepackage{color}
\usepackage{natbib}
\usepackage{algorithm,algorithmic}
\usepackage{graphicx,natbib,url}

\title{Non-Uniform Stochastic Average Gradient Method\\for Training Conditional Random Fields}
\author{
Mark Schmidt$^1$, Reza Babanezhad$^1$, Mohamed Osama Ahmed$^1$, \\
Aaron Defazio$^2$, Ann Clifton$^3$, 
Anoop Sarkar$^3$\\\\
$^1$University of British Columbia, $^2$Ambiata, $^3$Simon Fraser University}
\date{}

\def\norm#1{\|#1\|}
\newcommand{\fig}[2]{\includegraphics[width=#1\textwidth]{#2}}
\newtheorem{proposition}{Proposition}
\newtheorem{lemma}{Lemma}

\definecolor{red}{rgb}{1,0,0}

\global\long\def\stepsize{\gamma}

\begin{document}
\maketitle

\begin{abstract}
We apply stochastic average gradient (SAG) algorithms for training conditional random fields (CRFs). 
We describe a practical implementation that uses structure in the CRF gradient to reduce the memory requirement of this linearly-convergent stochastic gradient method, propose a non-uniform sampling scheme that substantially improves practical performance, and analyze the rate of convergence of the SAGA variant under non-uniform sampling. Our experimental results reveal that our method often significantly outperforms existing methods in terms of the training objective, and performs as well or better than optimally-tuned stochastic gradient methods in terms of test error.
\end{abstract}

\section{Introduction}

Conditional random fields (CRFs)~\citep{lafferty2001conditional} are a ubiquitous tool in natural language processing. They are used for part-of-speech tagging~\citep{phan2006}, semantic role labeling~\citep{Cohn05semanticrole}, topic modeling~\citep{Zhu:2010}, information extraction~\citep{peng2006information}, shallow parsing~\citep{sha2003shallow}, named-entity recognition~\citep{settles2004biomedical}, as well as a host of other applications in natural language processing and in other fields such as computer vision~\citep{nowozin2011structured}. Similar to  generative Markov random field (MRF) models, CRFs allow us to model probabilistic dependencies between output variables. 
The key advantage of discriminative CRF models is the ability to use a very high-dimensional feature set, without explicitly building a model for these features (as required by MRF models).
Despite the widespread use of CRFs, a major disadvantage of these models is that they can be very slow to train and
the time needed for numerical optimization in CRF models remains a bottleneck in many applications.

Due to the high cost of evaluating the CRF objective function on even a single training example, it is now common to train CRFs using stochastic gradient methods~\citep{vishwanathan2006accelerated}. These methods are advantageous over deterministic methods because on each iteration they only require computing the gradient of a single example (and not \emph{all} example as in deterministic methods). Thus, if we have a data set with $n$ training examples, the iterations of stochastic gradient methods are $n$ times faster than deterministic methods. However, the number of stochastic gradient iterations required might be very high. This has been studied in the optimization community, which considers the problem of finding the minimum number of iterations $t$ so that we can guarantee that we reach an accuracy of $\epsilon$, meaning that
\[
 f(w^t) - f(w^*) \leq \epsilon \textrm{, and } \norm{w^t - w^*}^2 \leq \epsilon,
\]
where $f$ is our training objective function, $w^t$ is our parameter estimate on iteration $t$, and $w^*$ is the parameter vector minimizing the training objective function. For strongly-convex objectives like $\ell_2$-regularized CRFs, stochastic gradient methods require $O(1/\epsilon)$ iterations~\citep{nemirovski2009robust}. 
This is in contrast to traditional deterministic methods which only require $O(\log(1/\epsilon))$ iterations~\citep{nesterov2004introductory}. However, this much lower number of iterations comes at the cost of requiring us to process the entire data set on each iteration.

For problems with a finite number of training examples,~\citet{roux2012stochastic} recently proposed the stochastic average gradient (SAG) algorithm which combines the advantages of deterministic and stochastic methods: it  only requires evaluating a single randomly-chosen training example on each iteration, and only requires $O(\log(1/\epsilon))$ iterations to reach an accuracy of $\epsilon$.
Beyond this faster convergence rate, the SAG method also allows us to address two issues that have traditionally frustrated users of stochastic gradient methods: \emph{setting the step-size} and \emph{deciding when to stop}. Implementations of the SAG method use both an adaptive step-size procedure and a cheaply-computable criterion for deciding when to stop. \citet{roux2012stochastic} show impressive empirical performance of the SAG algorithm for binary classification.

This is the first work to apply a SAG algorithm to train CRFs. We show that tracking marginals in the CRF can drastically reduce the SAG method's huge memory requirement. We also give a non-uniform sampling (NUS) strategy that adaptively estimates how frequently we should sample each data point, and we  show that the SAG-like algorithm of~\citet{defazio2014saga} converges under any NUS strategy while a particular NUS strategy achieves a faster rate. Our experiments compare the SAG algorithm with  a variety of competing deterministic, stochastic, and semi-stochastic methods on benchmark data sets for four common tasks: part-of-speech tagging, named entity recognition, shallow parsing, and optical character recognition.
Our results indicate that the SAG algorithm with NUS  often outperforms previous methods by an order of magnitude in terms of the training objective and, despite not requiring us to tune the step-size, performs as well or better than optimally tuned stochastic gradient methods in terms of the test error.

\section{Conditional Random Fields}

CRFs model the conditional probability of a structured output $y \in \mathcal{Y}$ (such as a sequence of labels) given an input $x \in \mathcal{X}$ (such as a sequence of words) based on features $F(x,y)$ and parameters $w$ using
\begin{equation}
\label{eq:logLinear}
 p(y|x,w) = \frac{\exp(w^TF(x,y))}{\sum_{y'}\exp(w^TF(x,y'))}.
\end{equation}
Given $n$ pairs $\{x_i,y_i\}$ comprising our training set, the standard approach to training the CRF is to minimize the $\ell_2$-regularized negative log-likelihood,
\begin{equation}
\label{eq:CRF}
 \min_w f(w) = \frac{1}{n}\sum_{i=1}^n -\log p(y_i|x_i,w) + \frac{\lambda}{2}\norm{w}^2,
\end{equation}
where $\lambda > 0$ is the strength of the regularization parameter. Unfortunately, evaluating $\log p(y_i|x_i,w)$ is expensive due to the summation over all possible configurations $y'$. For example, in chain-structured models the forward-backward algorithm is used to  compute $\log p(y_i|x_i,w)$ and its gradient. A second problem with solving~\eqref{eq:CRF} is that the number of training examples $n$ in applications is constantly-growing, and thus we would like to use methods that only require a few passes through the data set. 

\section{Related Work}

\citet{lafferty2001conditional} proposed an iterative scaling algorithm to solve problem~\eqref{eq:CRF}, but this proved to be inferior to generic deterministic optimization strategies like the limited-memory quasi-Newton algorithm L-BFGS~\citep{wallach2002efficient,sha2003shallow}. The bottleneck in these methods is that we must evaluate $\log p(y_i|x_i,w)$ and its gradient for all $n$ training examples on every iteration. This is  very expensive for problems where $n$ is very large, so to deal with this problem stochastic gradient methods were examined~\citep{vishwanathan2006accelerated,finkel-kleeman-manning:2008:ACLMain}. However, traditional stochastic gradient methods require $O(1/\epsilon)$ iterations rather than the much smaller  $O(\log(1/\epsilon))$ required by deterministic methods. 

There have been several attempts at improving the cost of deterministic methods or the convergence rate of stochastic methods. For example, the exponentiated gradient method of~\citet{collins2008exponentiated} processes the data online and only requires $O(\log(1/\epsilon))$ iterations to reach an accuracy of $\epsilon$ in terms of the dual objective. However, this does not guarantee good performance in terms of the primal objective or the weight vector. Although this method is highly-effective if $\lambda$ is very large, our experiments and the experiments of others show that the performance of online exponentiated gradient can degrade substantially if a small value of $\lambda$ is used (which may be required to achieve the best test error), see~\citet[Figures 5-6 and Table 3]{collins2008exponentiated} and~\citet[Figure 1]{blockFrankWolfe}. In contrast, SAG degrades more gracefully as $\lambda$ becomes small, even achieving a convergence rate faster than classic SG methods when $\lambda=0$~\citep{schmidt2013finite}.
\citet{lavergne2010practical} consider using multiple processors and vectorized computation to reduce the high iteration cost of quasi-Newton methods, but when $n$ is enormous these methods still have a high iteration cost.~\citet{friedlander2011hybrid} explore a hybrid deterministic-stochastic method that slowly grows the number of examples that are considered in order to achieve an $O(\log(1/\epsilon))$ convergence rate with a decreased cost compared to deterministic methods.

Below we state the convergence rates of different methods for training CRFs, including the fastest known rates for deterministic algorithms (like L-BFGS and accelerated gradient)~\citep{nesterov2004introductory}, stochastic algorithms (like [averaged] stochastic gradient and AdaGrad)~\citep{ghadimi2012optimal}, online exponentiated gradient, and SAG. Here $L$ is the Lipschitz constant of the gradient of the objective, $\mu$ is the strong-convexity constant (and we have $\lambda \leq \mu \leq L$), and $\sigma^2$ bounds the variance of the gradients.\\
\begin{center}
\begin{tabular}{llc}
Deterministic: & $O(n\sqrt{\frac{L}{\mu}}\log(1/\epsilon))$ & (primal)\\
Online EG & $O((n + \frac{L}{\lambda})\log(1/\epsilon))$ & (dual)\\
Stochastic & $O(\frac{\sigma^2}{\mu\epsilon}+\sqrt{\frac{L}{\mu\epsilon}})$ & (primal)\\
SAG & $O((n + \frac{L}{\mu})\log(1/\epsilon))$ & (primal)
\end{tabular}
\end{center}

\section{Stochastic Average Gradient}
\label{sec:SAG}

~\citet{roux2012stochastic} introduce the SAG algorithm, a simple method with the low iteration cost of stochastic gradient methods but that only requires $O(\log(1/\epsilon))$ iterations. To motivate this new algorithm, we write the classic gradient descent iteration as
\begin{equation}
w^{t+1} = w^t - \frac{\alpha}{n}\sum_{i=1}^n s_i^t,
\label{eq:SAG}
\end{equation}
where $\alpha$ is the step-size and at each iteration we set the `slope' variables $s_i^t$ to the gradient with respect to training example $i$ at $w^t$, so that $s_i^t = -\nabla \log p(y_i|x_i,w^t) + \lambda w^t$. The SAG algorithm uses this same iteration, but instead of updating $s_i^t$ for all $n$ data points on every iterations, it simply sets $s_i^t = -\nabla \log p(y_i|x_i,w^t) + \lambda w^t$ for \emph{one randomly chosen} data point and keeps the remaining $s_i^t$ at their value from the previous iteration. Thus the SAG algorithm is a randomized version of the gradient algorithm where we use the gradient of each example from the last iteration where it was selected.
The surprising aspect of the work of~\citet{roux2012stochastic} is that this simple \emph{delayed} gradient algorithm achieves a similar convergence rate to the classic full gradient algorithm despite the iterations being $n$ times faster.

\subsection{Implementation for CRFs}

Unfortunately, a major problem with applying~\eqref{eq:SAG} to CRFs is the requirement to store the $s_i^t$. While the CRF gradients $\nabla \log p(y_i|x_i,w^t)$ have a nice structure (see Section~\ref{sec:memory}), $s_i^t$ includes $\lambda w^t$ for some previous $t$, which is dense and unstructured. To get around this issue, instead of using~\eqref{eq:SAG} we use the following SAG-like update~\citep[Section 4]{roux2012stochastic}
\begin{align}
w^{t+1} & = w^t - \alpha(\frac{1}{m}\sum_{i=1}^n g_i^t + \lambda w^t)\nonumber\\
 & = w^t - \alpha(\frac{1}{m}d + \lambda w^t)\nonumber\\
& = (1 - \alpha\lambda)w^t - \frac{\alpha}{m}d,
 \label{eq:SAG2}
\end{align}
where $g_i^t$ is the value of $-\nabla \log p(y_i|x_i,w^k)$ for the last iteration $k$ where $i$ was selected and $d$ is the sum of the $g_i^t$ over all $i$. Thus, this update uses the exact gradient of the regularizer and only uses an approximation for the (structured) CRF log-likelihood gradients. Since we don't yet have any information about these log-likelihoods at the start, we initialize the algorithm by setting $g_i^0 = 0$. But to compensate for this, we track the number of examples seen $m$, and normalize $d$ by $m$ in the update (instead of $n$). In Algorithm~\ref{alg:SAG}, we summarize this variant of the SAG algorithm for training CRFs.\footnote{If we solve the problem for a sequence of regularization parameters, we can obtain better performance by warm-starting $g_i^0$, $d$, and $m$.}

\begin{algorithm}
\begin{algorithmic}[1]
    \REQUIRE $\{x_i,y_i\}$, $\lambda$, $w$, $\delta$
\STATE $m \leftarrow 0$, $g_i \leftarrow 0$ for $i=1,2,\dots, n$
\STATE $d\leftarrow0$, $L_g\leftarrow1$
  \WHILE{$m < n$ and $\norm{\frac{1}{n}d + \lambda w}_\infty \geq \delta$}
\STATE Sample $i$ from $\{1,2,\dots,n\}$
\STATE $f \leftarrow - \log p(y_i|x_i,w)$
\STATE $g \leftarrow - \nabla \log p(y_i|x_i,w)$
\IF{this is the first time we sampled $i$}
\STATE $m\leftarrow m+1$
\ENDIF
\\ Subtract old gradient $g_i$, add new gradient $g$:
\STATE $d \leftarrow d - g_i +  g$
\\ Replace old gradient of example $i$:
\STATE $g_i \leftarrow g$
\IF{$\norm{g_i}^2 > 10^{-8}$}
\STATE $L_g \leftarrow $lineSearch$(x_i,y_i,f,g_i,w,L_g)$
\ENDIF
\STATE $\alpha \leftarrow 1/(L_g + \lambda)$
  \STATE $w \leftarrow (1-\alpha\lambda)w - \frac{\alpha}{m}d$
\STATE $L_g \leftarrow L_g \cdot 2^{-1/n}$
\ENDWHILE
\end{algorithmic}
\caption{SAG algorithm for training CRFs}
\label{alg:SAG}
\end{algorithm}

In many applications of CRFs the $g_i^t$ are very sparse, and we would like to take advantage of this as in stochastic gradient methods. Fortunately, we can implement~\eqref{eq:SAG2} without using dense vector operations by using the representation $w^t = \beta^tv^t$ for a scalar $\beta^t$ and a vector $v^t$, and using `lazy updates' that apply $d$ repeatedly to an individual variable when it is needed~\citep{roux2012stochastic}. 

Also following~\citet{roux2012stochastic}, we set the step-size to $\alpha = 1/L$, where $L$ is an approximation to the maximum Lipschitz constant of the gradients. This is the smallest number $L$ such that
\begin{equation}
\norm{\nabla f_i(w) - \nabla f_i(v)} \leq L\norm{w-v},
\label{eq:L}
\end{equation}
for all $i$, $w$, and $v$. This quantity is a bound on how fast the gradient can change as we change the weight vector. The Lipschitz constant with respect to the gradient of the regularizer is simply $\lambda$. This gives $L \leq L_g + \lambda$, where $L_g$ is the Lipschitz constant of the gradient of the log-likelihood. Unfortunately, $L_g$ depends on the covariance of the CRF and is typically too expensive to compute. To avoid this computation, as in~\citet{roux2012stochastic} we approximate $L_g$ in an online fashion using the standard backtracking line-search given by Algorithm~\ref{alg:LS}~\citep{beck2009fast}. The test used in this algorithm is faster than testing~\eqref{eq:L}, since it uses function values (which only require the forward algorithm for CRFs) rather than gradient values (which require the forward and backward steps). Algorithm~\ref{alg:LS} monotonically increases $L_g$, but we also slowly decrease it in Algorithm~\ref{alg:SAG} in order to allow the possibility that we can use a more aggressive step-size as we approach the solution.
\begin{algorithm}[H]
\begin{algorithmic}[1]
    \REQUIRE $x_i,y_i,f,g_i,w,L_g$.
  \STATE $f' = -\log p(y_i|x_i,w - \frac{1}{L_g}g_i)$
\WHILE{$f' \geq f - \frac{1}{2L_g}\norm{g_i}^2$}
\STATE $L_g = 2L_g$
\STATE $f' = -\log p(y_i|x_i,w - \frac{1}{L_g}g_i)$
\ENDWHILE
\RETURN $L_g$.
\end{algorithmic}
\caption{Lipschitz line-search algorithm}
\label{alg:LS}
\end{algorithm}
\vspace{-10pt}
Since the solution is the only stationary point, we must have $\nabla f(w^t) = 0$ at the solution. Further, the value $\frac{1}{n}d + \lambda w^t$ converges to $\nabla f(w^t)$ so we can use the size of this value to decide when to stop the algorithm (although we also require that $m=n$ to avoid premature stopping before we have seen the full data set).  This is in contrast to classic stochastic gradient methods, where the step-size must go to zero and it is therefore difficult to decide if the algorithm is close to the optimal value or if we simply require a small step-size to continue making progress.

\subsection{Reducing the Memory Requirements}
\label{sec:memory}

Even if the gradients $g_i^t$ are not sparse, we can often reduce the memory requirements of Algorithm~\ref{alg:SAG} because it is known that the CRF gradients only depend on $w$ through marginals of the features. Specifically, the gradient of the log-likelihood under model~\eqref{eq:logLinear} with respect to feature $j$ is given by
\begin{align*}
\nabla_j \log p(y|x,w) & = F_j(x,y) - \frac{\sum_{y'}\exp(F(x,y'))F_j(x,y')}{\sum_{y'}\exp(F(x,y'))}\\
& = F_j(x,y) - \sum_{y'}p(y'|x,w)F_j(x,y')\\
& = F_j(x,y) - \mathbb{E}_{y'|x,w}[F_j(x,y')]
\end{align*}
Typically, each feature $j$ only depends on a small `part' of $y$. For example, we typically include features of the form $F_j(x,y) = F(x)\mathbb{I}[y_k = s]$ for some function $F$, where $k$ is an element of $y$ and $s$ is a discrete state that $y_k$ can take. In this case, the gradient can be written in terms of the marginal probability of element $y_k$ taking state $s$,
\begin{align*}
\nabla_j \log p(y|x,w) & = F(x)\mathbb{I}[y_k = s] - \mathbb{E}_{y'|x,w}[F(x)\mathbb{I}[y_k = s]]\\
& = F(x)(\mathbb{I}[y_k = s] - \mathbb{E}_{y'|x,w}[\mathbb{I}[y_k = s])\\
& = F(x)(\mathbb{I}[y_k = s] - p(y_k=s|x,w)).
\end{align*}
Notice that Algorithm~\ref{alg:SAG} only depends on the old gradient through its difference with the new gradient (line 10), which in this example gives
\begin{align*}
\nabla_j \log p(y|x,w) - \nabla_j \log p(y|x,w_\text{old}) = 
 F(x)(p(y_k=s|x,w_{\text{old}}) - p(y_k=s|x,w)),
\end{align*}
where $w$ is the current parameter vector and $w_{\text{old}}$ is the old parameter vector. Thus, to perform this calculation the only thing we need to know about $w_{\text{old}}$ is the unary marginal $p(y_k=s|x,w_\text{old})$, which will be \emph{shared} across features that only depend on the event that $y_k=s$. Similarly, features that depend on pairs of values in $y$ will need to store pairwise marginals, $p(y_k=s,y_k'=s'|x,w_\text{old})$.
For general pairwise graphical model structures, the memory requirements to store these marginals will thus be $O(VK + EK^2)$, where $V$ is the number of vertices and $E$ is the number of edges. 
This can be an enormous reduction since \emph{it does not depend on the number of features}. 
 Further, since computing these marginals is a by-product of computing the gradient, this potentially-enormous reduction in the memory requirements comes at no extra computational cost.

\section{Non-Uniform Sampling}
\label{sec:NUS}

Recently, several works show that we can improve the convergence rates of randomized optimization algorithms by using non-uniform sampling (NUS) schemes. This includes  randomized Kaczmarz~\citep{strohmer2009randomized}, randomized coordinate descent~\citep{nesterov2012efficiency}, and stochastic gradient methods~\citep{SGkaczmarz}.
The key idea behind all of these NUS strategies is to \emph{bias the sampling towards the Lipschitz constants of the gradients}, so that gradients that change quickly get sampled more often and gradients that change slowly get sampled less often. Specifically, we maintain a Lipschitz constant $L_i$ for each training example $i$ and, instead of the usual sampling strategy $p_i = 1/n$, we bias towards the distribution $p_i = L_i/\sum_j L_j$. In these various contexts, NUS allows us to improve the dependence on the values $L_i$ in the convergence rate, since the NUS methods depend on $\bar{L} = (1/n)\sum_j L_j$, which may be substantially smaller than the usual dependence  on $L = \max_j\{ L_j\}$.~\citet{schmidt2013finite} argue that faster convergence rates might be achieved with NUS for SAG since it allows a larger step size $\alpha$ that depends on $\bar{L}$ instead of $L$.\footnote{An interesting difference between the SAG update with NUS and NUS for stochastic gradient methods is that the SAG update does not seem to need to decrease the step-size for frequently-sampled examples (since the SAG update does not rely on using an unbiased gradient estimate).}


 The scheme for SAG proposed by~\citet[][Section 5.5]{schmidt2013finite} uses a fairly complicated adaptive NUS scheme and step-size, but the key ingredient is estimating each constant $L_i$ using Algorithm~\ref{alg:LS}. Our experiments show this method often already improves on state of the art methods for training CRFs by a substantial margin, but we found we could obtain improved performance for training CRFs using the following simple NUS scheme for SAG: as in~\citet{SGkaczmarz}, with probability $0.5$ choose $i$ uniformly and with probability $0.5$ sample $i$ with probability $L_i/(\sum_j L_j)$ (restricted to the examples we have previously seen).\footnote{\citet{SGkaczmarz} analyze the basic stochastic gradient method and thus require $O(1/\epsilon)$ iterations.} We also use a step-size of $\alpha = \frac{1}{2}\left(1/L + 1/\bar{L}\right)$, since the  faster convergence rate with NUS is due to the ability to use a larger step-size than $1/L$. 
This simple step-size and sampling scheme contrasts with the more complicated choices described by~\citet[][Section 5.5]{schmidt2013finite}, that make the degree of non-uniformity grow with the number of examples seen $m$. 
This prior work initializes each $L_i$ to $1$, and updates $L_i$ to $0.5L_i$ each subsequent time an example is chosen. In the context of CRFs, this leads to a large number of expensive backtracking iterations. To avoid this, we initialize $L_i$ with $0.5\bar{L}$ the first time an example is chosen, and decrease $L_i$ to $0.9L_i$ each time it is subsequently chosen. Allowing the $L_i$ to decrase seems crucial to obtaining the best practical performance of the method, as it allows the algorithm to take bigger step sizes if the values of $L_i$ are small near the solution.

\subsection{Convergence Analysis under NUS}

\citet{schmidt2013finite} give an intuitive but non-rigorous motivation for using NUS in SAG. More recently,~\citet{xiao2014proximal} show that NUS gives a dependence on $\bar{L}$ in the context of a related algorithm that uses occasional full passes through the data (which substantially simplifies the analysis). 
Below, we analyze a NUS extension of the SAGA algorithm of~\citet{defazio2014saga}, which does not require full passes through the data and has similar performance to SAG in practice but is much easier to analyze.
\begin{proposition}
Let the sequences $\{w^t\}$ and $\{s_j^t\}$ be defined by
\begin{align*}
w^{t+1} & = w^t - \alpha\left[\frac{1}{np_{j_t}}  ( \nabla f_{j_t}(w^t) - s_{j_t}^t) + \frac{1}{n}\sum_{i=1}^ns_i^t\right],\\
s_j^{t+1} & = \begin{cases}\nabla f_{r_t}(w^t) & \text{if $j = r_t$,}\\
s_j^t & \text{otherwise.}
\end{cases}
\end{align*}
where $j_t$ is chosen with probability $p_j$.

(a) If $r_t$ is set to $j_t$, then with $\alpha = \frac{np_\text{min}}{4L + n\mu}$ we have
\[
\mathbb{E}[\norm{w^t-w^*}^2] \leq \left(  1 - \mu\alpha \right)^t \left[\norm{x^0 - x^*} + C_a\right],
\]
where $p_\text{min} = \min_i\{p_i\}$ and
\[
C_a =  \frac{2p_\text{min}}{(4L+n\mu)^2}\sum_{i=1}^n\frac{1}{p_i}\norm{\nabla f_i(x^0) - \nabla f_i(x^*)}^2.
\]
(b) If $p_j = \frac{L_j}{\sum_{i=1}^nL_i}$ and $r_t$ is chosen uniformly at random, then with $\alpha = \frac{1}{4\bar{L}}$ we have
\[
\mathbb{E}[\norm{w^t-w^*}^2] \leq \left(  1 - \min\left\{\frac{1}{3n},\frac{\mu}{8\bar{L}}\right\} \right)^t \left[\norm{x^0 - x^*} + C_b\right],
\]
where:
\[
C_b =  \frac{n}{2\bar{L}}\left[f(x^0) - f(x^*)\right]
\]
\end{proposition}
 This result (which we prove in Appendix A and B) shows that SAGA has (a) a linear convergence rate for any NUS scheme where $p_i > 0$ for all $i$, and (b) a rate depending on $\bar{L}$ by sampling proportional to the Lipschitz constants and also generating a uniform sample. However, (a) achieves the fastest rate when $p_i = 1/n$ while (b) requires two samples on each iteration. We were not able to show a faster rate using only one sample on each iteration as used in our implementation.

\subsection{Line-Search Skipping}

To reduce the number of function evaluations required by the NUS strategy, we also explored a line-search \emph{skipping} strategy. The general idea is to consider skipping the line-search for example $i$ if the line-search criterion was previously satisfied for example $i$ without backtracking. Specifically, if the line-search criterion was satisfied $\xi$ consecutive times for example $i$ (without backtracking), then we do not do the line-search on the next $2^{\xi-1}$ times example $i$ is selected (we also do not multiply $L_i$ by $0.9$ on these iterations). This drastically reduces the number of function evaluations required in the later iterations.

\section{Experiments}
\label{sec:expts}

We compared a wide variety of approaches on four CRF training tasks: the optical character recognition (OCR) dataset of~\citet{tasker2003max}, the CoNLL-2000 shallow parse chunking dataset,\footnote{http://www.cnts.ua.ac.be/conll2000/chunking} the CoNLL-2002 Dutch named-entity recognition dataset,\footnote{http://www.cnts.ua.ac.be/conll2002/ner} and a part-of-speech (POS) tagging task using the Penn Treebank Wall Street Journal data (POS-WSJ). The optimal character recognition dataset labels the letters in images of words. Chunking segments a sentence into syntactic chunks by tagging each sentence token with a chunk tag corresponding to its constituent type (e.g., `NP', `VP', etc.) and location (e.g., beginning, inside, ending, or outside any constituent). We use standard n-gram and POS tag features~\citep{sha2003shallow}. For the named-entity recognition task, the goal is to identify named entities and correctly classify them as persons, organizations, locations, times, or quantities. We again use standard n-gram and POS tag features, as well as word shape features over the case of the characters in the token.
The POS-tagging task assigns one of 45 syntactic tags to each token in each of the sentences in the data. For this data, we follow the standard division of the WSJ data given by~\citet{Collins:2002}, using sections 0-18 for training, 19-21 for development, and 22-24 for testing. We use the standard set of features following~\citet{ratnaparkhi1996maximum} and~\citet{Collins:2002}: n-gram, suffix, and shape features. As is common on these tasks, our pairwise features do not depend on $x$.

On these datasets we compared the performance of a set of competitive methods, including five variants on classic stochastic gradient methods: \emph{Pegasos} which is a standard stochastic gradient method with a step-size of $\alpha=\eta/\lambda t$ on iteration $t$~\citep{shalev2007pegasos},\footnote{We also tested \emph{Pegasos} with averaging but it always performed worse than the non-averaged version.} a basic stochastic gradient (\emph{SG}) method where we use a constant $\alpha = \eta$, an averaged stochastic gradient (\emph{ASG}) method where we use a constant step-size $\alpha = \eta$ and average the iterations,\footnote{We also tested \emph{SG} and \emph{ASG} with decreasing step-sizes of either $\alpha_t=\eta/\sqrt{t}$ or $\alpha_t=\eta/(\delta+t)$, but these gave worse performance than using a constant step size.} \emph{AdaGrad} where we use the per-variable $\alpha_j = \eta/(\delta + \sqrt{\sum_{i=1}^t\nabla_j \log p(y_i|x_i,w^i)^2})$ and the proximal-step with respect to the $\ell_2$-regularizer~\citep{duchi2010adaptive}, and stochastic meta-descent (\emph{SMD}) where we initialize with $\alpha_j = \eta$ and dynamically update the step-size~\citep{vishwanathan2006accelerated}.
Since setting the step-size is a notoriously hard problem when applying stochastic gradient methods, we let these classic stochastic gradient methods cheat by choosing the $\eta$ which gives the best performance among powers of $10$ on the training data (for SMD we additionally tested the four choices among the paper and associated code of~\citet{vishwanathan2006accelerated}, and we found $\delta=1$ worked well for \emph{AdaGrad}).\footnote{Because of the extra implementation effort required to implement it efficiently, we did not test SMD on the POS dataset, but we do not expect it to be among the best performers on this data set.}
Our comparisons also included a deterministic \emph{L-BFGS} algorithm~\citep{schmidt2005minfunc} and the \emph{Hybrid} L-BFGS/stochastic algorithm of \citet{friedlander2011hybrid}. We also included the online exponentiated gradient (\emph{OEG}) method~\citep{collins2008exponentiated}, and we followed the heuristics in the author's code.\footnote{Specifcially, for OEG we proceed through a random permutation of the dataset on the first pass through the data, we perform a maximum of $2$ backtracking iterations per example on this first pass (and $5$ on subsequent passes), we initialize the per-sample step-sizes to $0.5$ and divide them by $2$ if the dual objective does not increase (and multiply them by $1.05$ after processing the example), and to initialize the dual variables we set parts with the correct label from the training set to $3$ and parts with the incorrect label to $0$.}
Finally, we included the \emph{SAG} algorithm as described in Section 4, the \emph{SAG-NUS} variant of~\citet{schmidt2013finite}, and our proposed \emph{SAG-NUS*} strategy from  Section~\ref{sec:NUS}.\footnote{We also tested \emph{SG} with the proposed NUS scheme, but the performance was similar to the regular SG method. This is consistent with the analysis of~\citet[Corollary 3.1]{SGkaczmarz} showing that NUS for regular \emph{SG} only improves the non-dominant term.} 
We also tested SAGA variants of each of the SAG algorithms, and found that they gave very similar performance. All methods (except OEG) were initialized at zero.

\begin{figure}
\begin{center}
 \fig{.35}{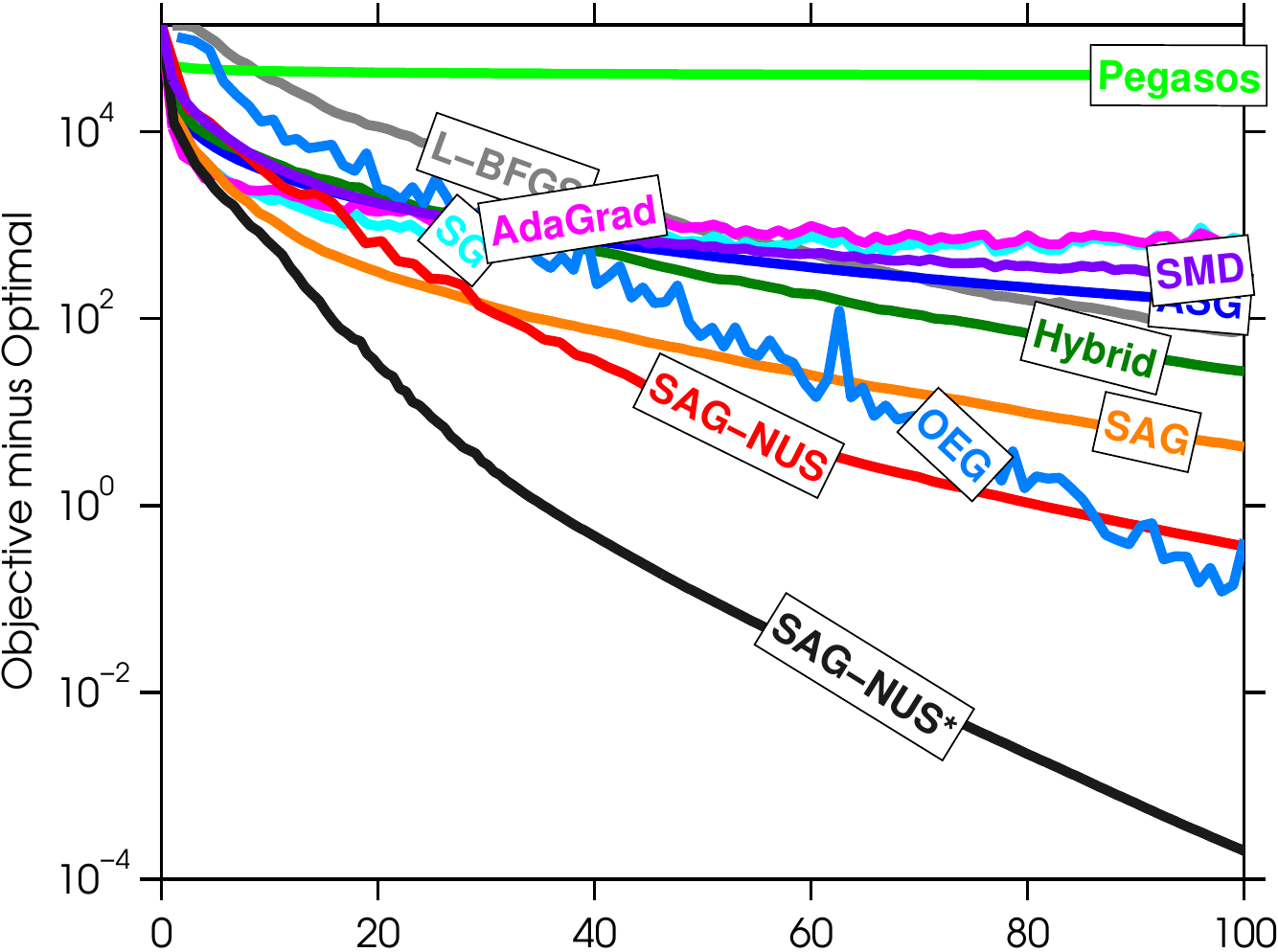}
\fig{.35}{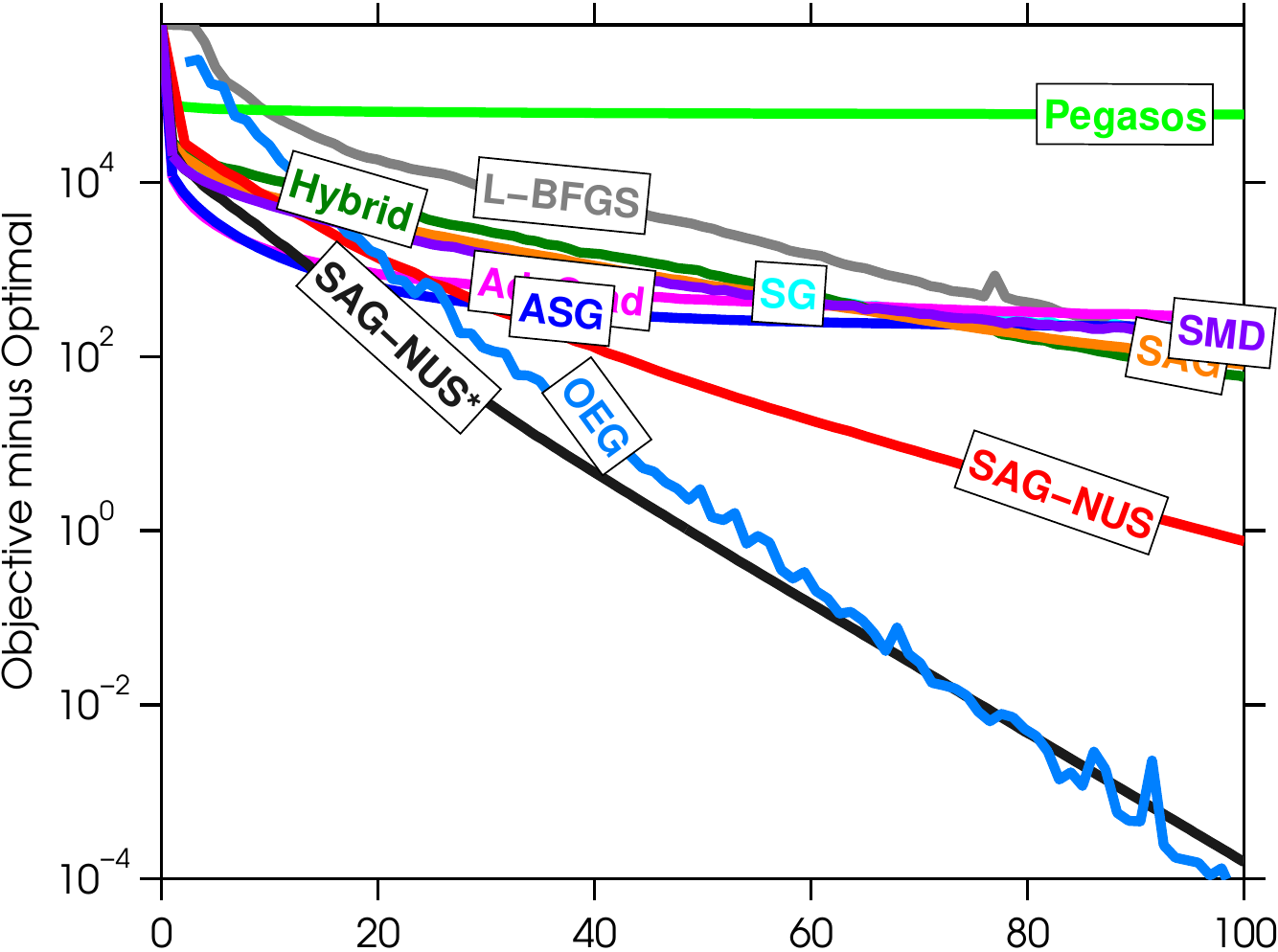}\\
\fig{.35}{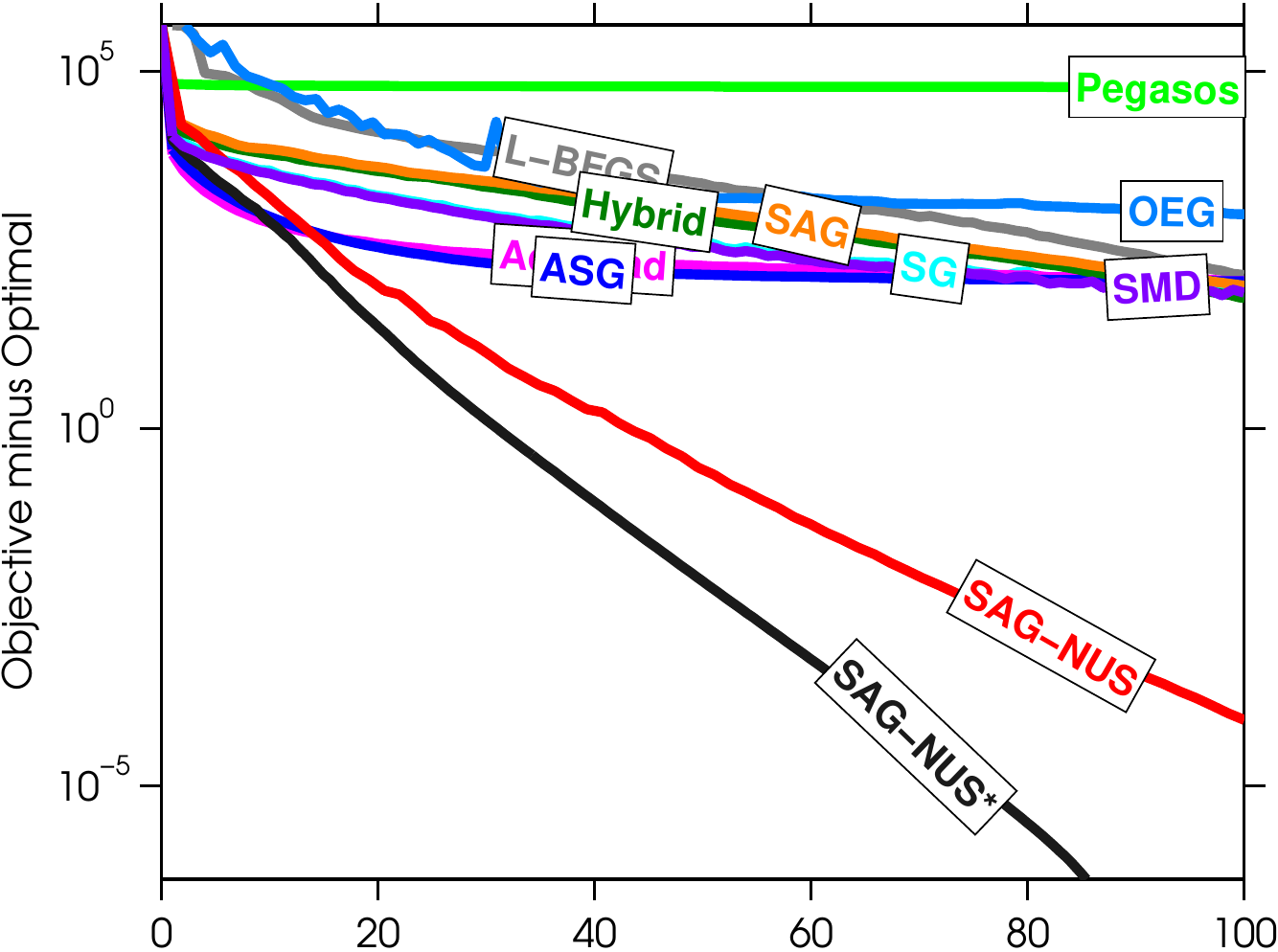}
\fig{.35}{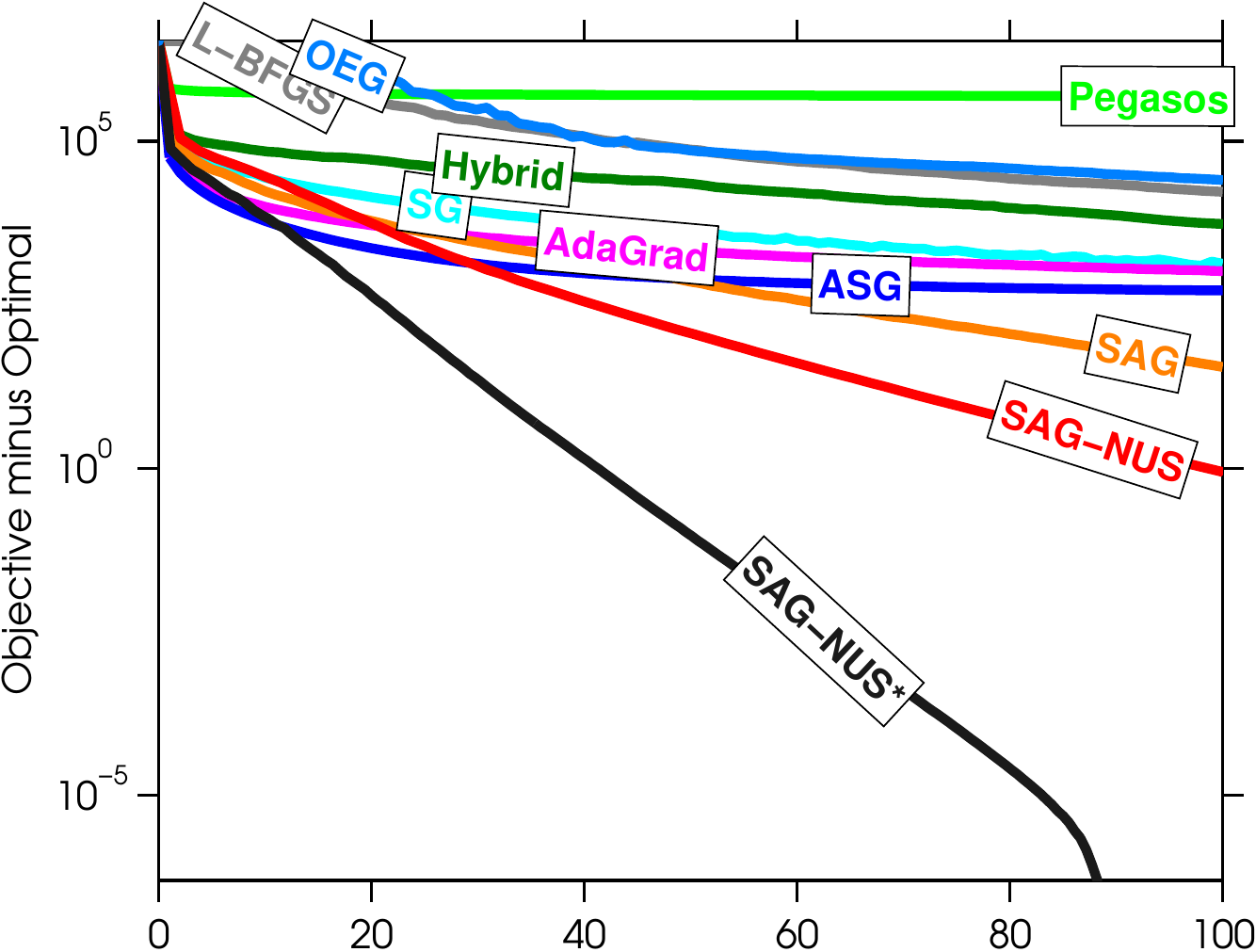} 
\end{center}
\caption{Objective minus optimal objective value against effective number of passes for different deterministic, stochastic, and semi-stochastic optimization strategies. Top-left: OCR, Top-right: CoNLL-2000, bottom-left: CoNLL-2002, bottom-right: POS-WSJ.}
\label{fig:trainPass}
\end{figure}

\begin{figure}[h]
\begin{center}
 \fig{.35}{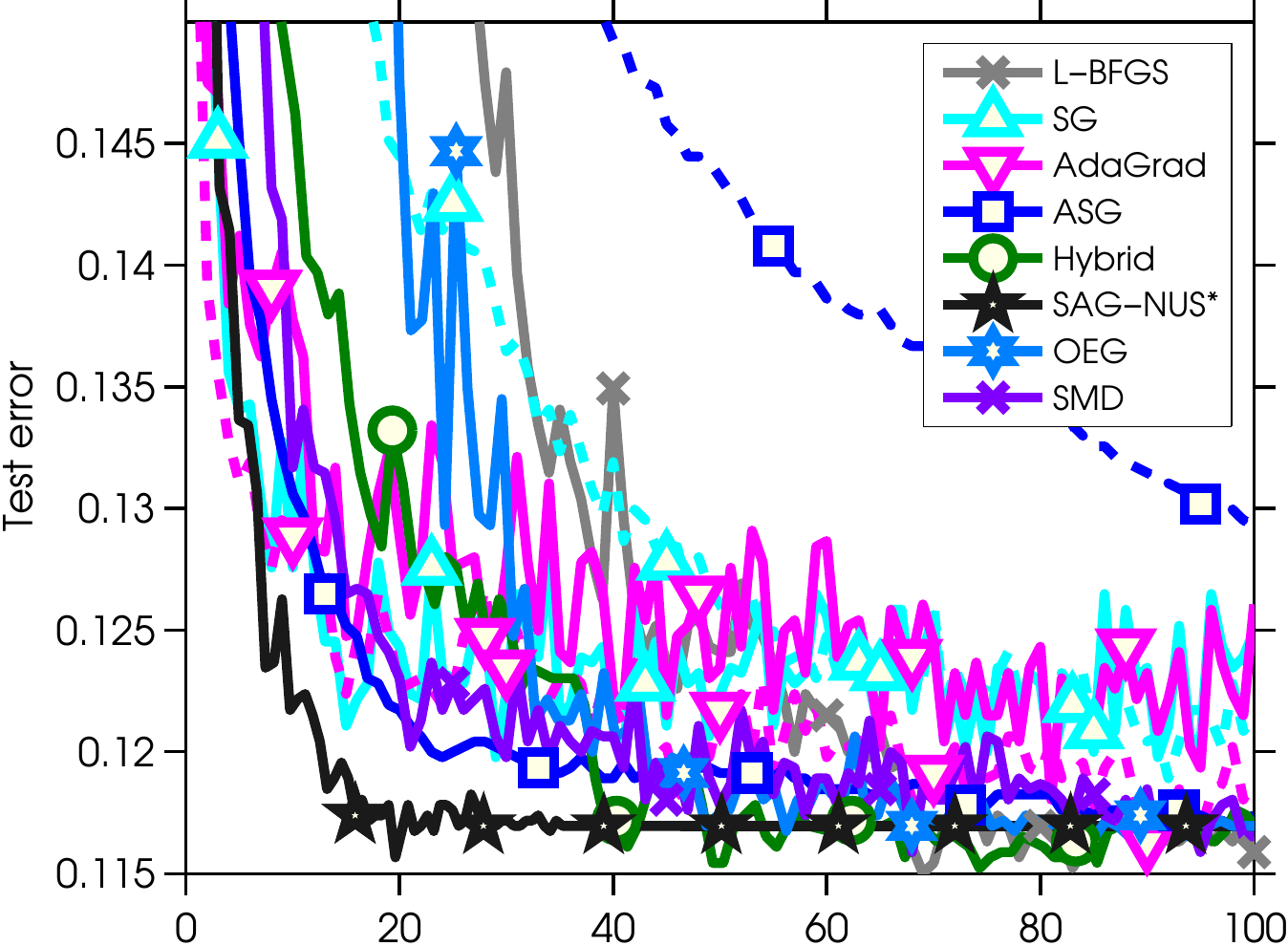}
\fig{.35}{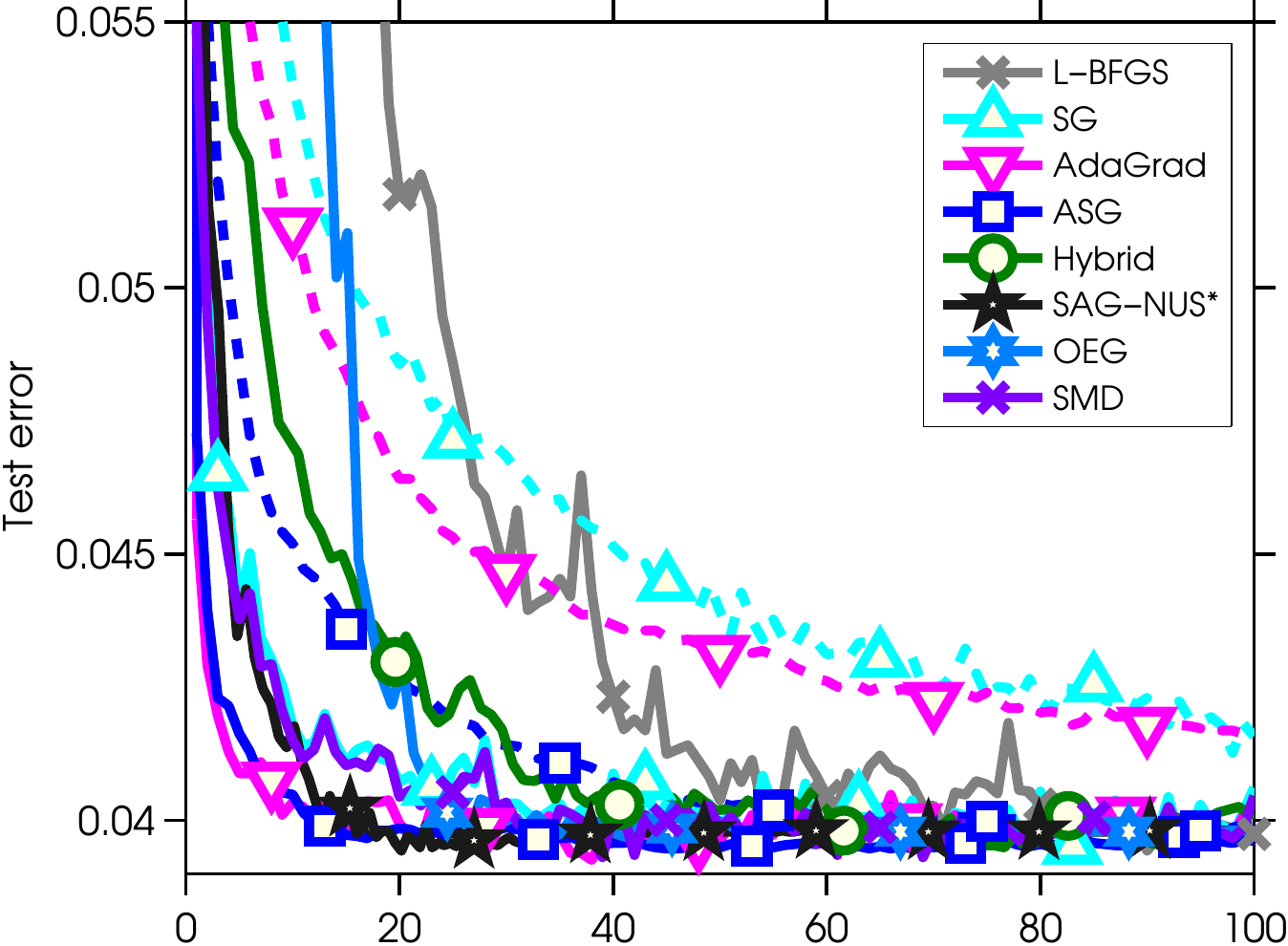}\\
\fig{.35}{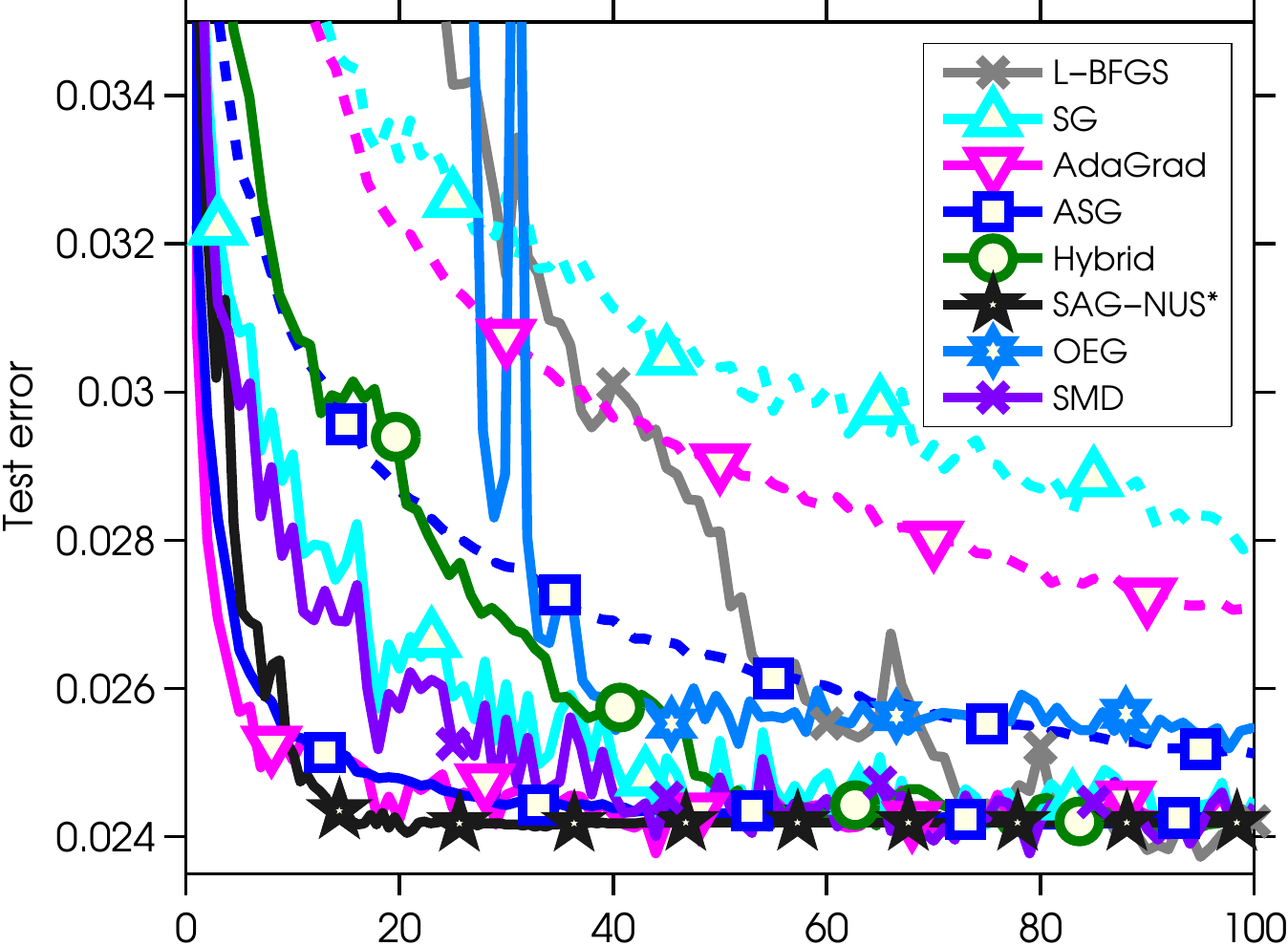}
\fig{.35}{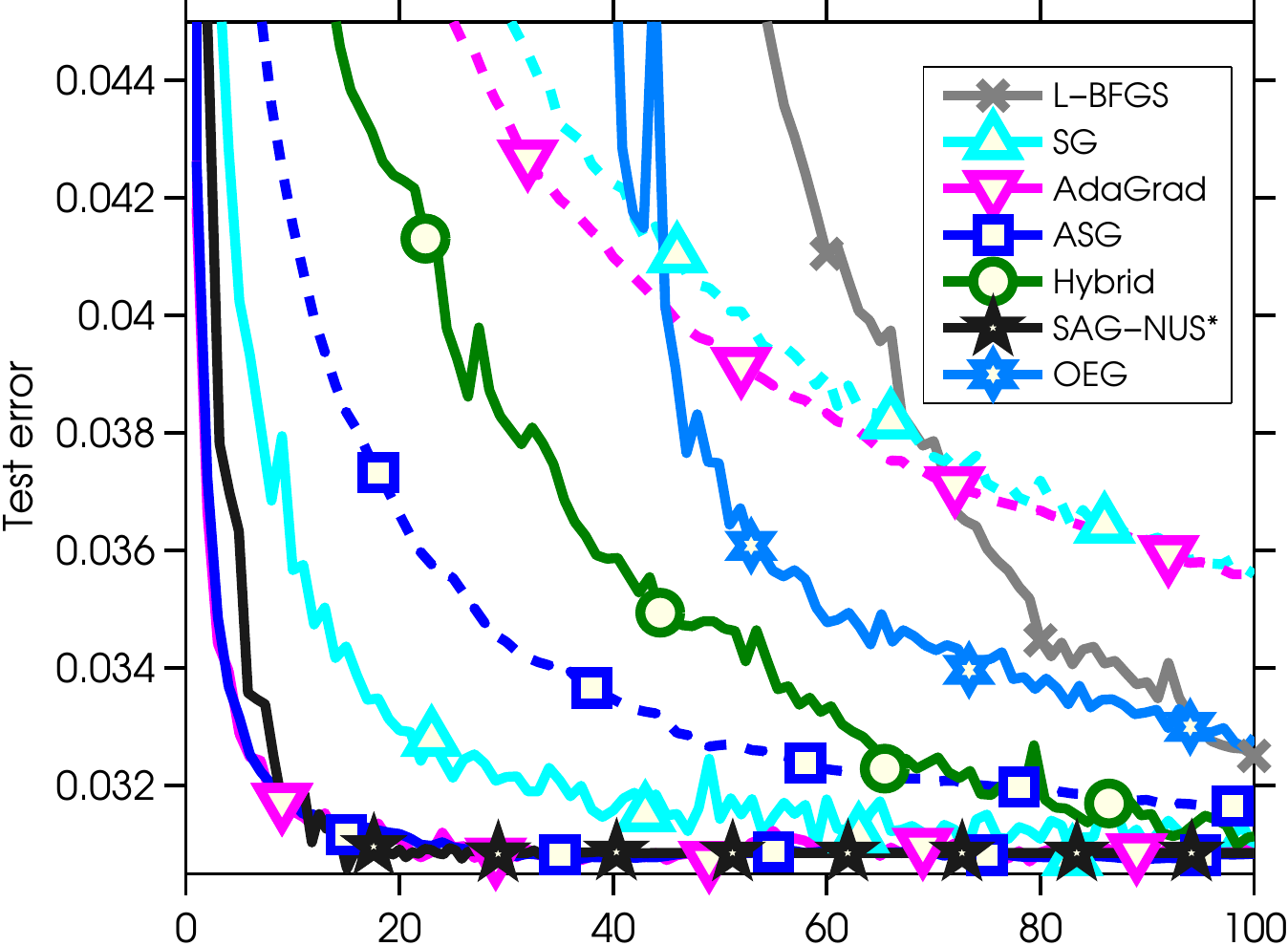} 
\end{center}
\caption{Test error against effective number of passes for different deterministic, stochastic, and semi-stochastic optimization strategies (this figure is best viewed in colour). Top-left: OCR, Top-right: CoNLL-2000, bottom-left: CoNLL-2002, bottom-right: POS-WSJ. The dotted lines show the performance of the classic stochastic gradient methods when the optimal step-size is not used. \emph{Note that the performance of all classic stochastic gradient methods is much worse when the optimal step-size is not used}, whereas the SAG methods have an adaptive step-size so are not sensitive to this choice.}
\label{fig:testPass} 
\end{figure}

Figure~\ref{fig:trainPass} shows the result of our experiments on the training objective and Figure~\ref{fig:testPass} shows the result of tracking the test error. Here we measure the number of `effective passes', meaning $(1/n)$ times the number of times we performed the bottleneck operation of computing $\log p(y_i|x_i,w)$ and its gradient. This is an implementation-independent way to compare the convergence of the different algorithms (most of whose runtimes differ only by a small constant), but we have included the performance in terms of runtime in Appendix E. For the different \emph{SAG} methods that use a line-search we count the extra `forward' operations used by the line-search as full evaluations of $\log p(y_i|x_i,w)$ and its gradient, even though these operations are cheaper because they do not require the backward pass nor computing the gradient. In these experiments we used $\lambda = 1/n$, which yields a value close to the optimal test error across all data sets.
 The objective is strongly-convex and thus has a unique minimum value. We approximated this value by running L-BFGS for up to $1000$ iterations, which always gave a value of $w$ satisfying $\norm{\nabla f(w)}_\infty \leq 1.4 \times 10^{-7}$, indicating that this is a very accurate approximation of the true solution. In the test error plots, we have excluded the \emph{SAG} and \emph{SAG-NUS} methods to keep the plots interpretable (while Pegasos does not appear becuase it performs very poorly), but  Appendx C includes these plots with all methods added. In the test error plots, we have also plotted as dotted lines the performance of the classic stochastic gradient methods when the second-best step-size is used.

We make several observations based on these experiments:
\begin{itemize}\itemsep 0pt
\item \emph{SG} outperformed \emph{Pegasos}. Pegasos is known to move exponentially away from the solution in the early iterations~\citep{bach2011non}, meaning that $\norm{w^t-w^*} \geq \rho^t\norm{w^0-w^*}$ for some $\rho > 1$, while SG moves exponentially towards the solution ($\rho < 1$) in the early iterations~\citep{nedic2001convergence}.
\item \emph{ASG} outperformed \emph{AdaGrad} and \emph{SMD} (in addition to \emph{SG}). ASG methods are known to achieve the same asymptotic efficiency as an optimal stochastic Newton method~\citep{polyak1992acceleration}, while AdaGrad and SMD can be viewed as approximations to a stochastic Newton method. ~\citet{vishwanathan2006accelerated} did not compare to ASG, because applying ASG to large/sparse data requires the recursion of~\citet{xu2010fastASG}.
\item \emph{Hybrid} outperformed \emph{L-BFGS}. The hybrid algorithm processes fewer data points in the early iterations, leading to cheaper iterations.
\item None of the three algorithms \emph{ASG}/\emph{Hybrid}/\emph{SAG} dominated the others: the relative ranks of these methods changed based on the data set and whether we could choose the optimal step-size.
\item \emph{OEG} performed very well on the first two datasets, but was less effective on the second two. By experimenting with various initializations, we found that we could obtain much better performance with OEG on these two datasets. We report these results in the Appendix D, although Appendix E shows that OEG was less competitive in terms of runtime.
\item Both \emph{SAG-NUS} methods outperform all other methods (except OEG) by a substantial margin based on the training objective, and are always among the best methods in terms of the test error. Further, our proposed \emph{SAG-NUS*} always outperforms \emph{SAG-NUS}.
\end{itemize}

On three of the four data sets, the best classic stochastic gradient methods (\emph{AdaGrad} and \emph{ASG}) seem to reach the optimal test error with a similar speed to the \emph{SAG-NUS*} method, although they require many passes to reach the optimal test error on the OCR data. Further, we see that the good test error performance of the \emph{AdaGrad} and \emph{ASG} methods \emph{is very sensitive to choosing the optimal step-size}, as the methods perform much worse if we don't use the optimal step-size (dashed lines in Figure~\ref{fig:testPass}). 
In contrast, SAG uses an adaptive step-size and has virtually identical performance even if  the initial value of $L_g$ is too small by several orders of magnitude (the line-search quickly increases $L_g$ to a reasonable value on the first training example, so the dashed black line in Figure 2 would be on top of the solid line).

To quantify the memory savings given by the choices in Section~\ref{sec:SAG}, below we report the size of the memory required for these datasets under different memory-saving strategies divided by the memory required by the naive SAG algorithm.
\emph{Sparse} refers to only storing non-zero gradient values, \emph{Marginals} refers to storing all unary and pairwise marginals, and \emph{Mixed} refers to storing node marginals and the gradient with respect to pairwise features (recall that the pairwise features do not depend on $x$ in our models).
\begin{center}
\begin{tabular}{llll}
Dataset & Sparse & Marginals & Mixed \\
OCR & $7.8 \times 10^{-1}$ & $1.1 \times 10^0$ & $2.1 \times 10^{-1}$\\
CoNLL-2000 & $4.8 \times 10^{-3}$ & $7.0 \times 10^{-3}$ & $6.1 \times 10^{-4}$\\
CoNLL-2002 & $6.4 \times 10^{-4}$ & $3.8 \times 10^{-4}$ & $7.0 \times 10^{-5}$\\
POS-WJ & $1.3 \times 10^{-3}$ & $5.5 \times 10^{-3}$ & $3.6 \times 10^{-4}$\\
\end{tabular}
\end{center}

\section{Discussion}


Due to its memory requirements, it may be difficult to apply the SAG algorithm for natural language applications involving complex features that depend on a large number of labels.
However, grouping training examples into mini-batches can also reduce the memory requirement (since only the gradients with respect to the mini-batches would be needed).  An alternative strategy for reducing the memory is to use the algorithm of~\citet{johnson2013accelerating} or~\citet{zhang2013linear}. These require evaluating the chosen training example twice on each iteration, and occasionally require full passes through the data, but do not have the memory requirements of SAG (in our experiments, these performed similar to or slightly worse than running SAG at half speed).

We believe linearly-convergent stochastic gradient algorithms with non-uniform sampling could give a substantial performance improvement in a large variety of CRF training problems, and we emphasize that the method likely has extensions beyond what we have examined. For example, we have focused on the case of $\ell_2$-regularization but for large-scale problems there is substantial interest in using $\ell_1$-regularization CRFs~\citep{tsuruoka-tsujii-ananiadou:2009:ACLIJCNLP,lavergne2010practical,zhou-qiu-huang:2011:IJCNLP-2011}. Fortunately, such non-smooth regularizers can be handled with a proximal-gradient variant of the method, see~\citet{defazio2014saga}.
While we have considered chain-structured data the algorithm applies to general graph structures, and any method for computing/approximating the marginals could be adopted. 
Finally, the SAG algorithm could be modified to use multi-threaded computation as in the algorithm of~\citet{lavergne2010practical}, and indeed might be well-suited to massively distributed parallel implementations.


\section*{Acknowledgments}

We would like to thank the anonymous reviewers as well as Simon Lacoste-Julien for their helpful comments. This research was supported by the Natural Sciences and Engineering Research Council of Canada (RGPIN 262313, RGPAS 446348, and CRDPJ 412844 which was made possible by a generous contribution from The Boeing Company and AeroInfo Systems). Travel support to attend the conference was provided by the Institute for Computing, Information and Cognitive Systems (ICICS).

\appendix

\section*{Appendx A: Proof of Part (a) of Proposition 1}

In this section we consider the minimization problem
\[
\min_x f(x) = \frac{1}{n}\sum_{i=1}^nf_i(x),
\]
where each $f_i'$ is $L$-Lipschitz continuous and each $f_i$ is $\mu$-strongly-convex. We will define Algorithm~2, a variant of SAGA, by the sequences $\{x^k\}$, $\{\nu_k\}$, and $\{\phi_j^k\}$ given by
\begin{align*}
    \nu_k & =\frac{1}{np_j}[f'_{j_k}(x^k)-f'_{j_k}(\phi_j^k)]+\frac{1}{n}\sum_{i=1}^{n}f'_i(\phi_i^k),\\
    x^{k+1} &=x^k-\frac{1}{\eta}\nu_k,\\
    \phi_{j}^{k+1}& =\begin{cases}f_{j_k}'(x^k) & \text{if  $j = j_k$,}\\
\phi_j^k & \text{otherwise,}\end{cases}
\end{align*}
where $j_k = j$ with probability $p_j$. In this section we'll use the convention that  $x=x^k$, that $\phi_j=\phi_j^k$, and that $x^*$ is the minimizer of $f$. We first show that $\nu_k$ is an unbiased gradient estimator and derive a bound on its variance.
\begin{lemma}
\label{lemma:variance}
We have  $\mathbb E[\nu_k]=f'(x^k)$ and subsequently
 \[
\mathbb E\|\nu_k\|^2 \leq 2\mathbb E\|\frac{1}{np_j}[f'_j(x)-f'_j(x^*)]\|^2+2\mathbb E\|\frac{1}{np_j}[f'_j(\phi_j)-f'_j(x^*)]\|^2.
\]
\end{lemma}
\begin{proof}
We have
\begin{align*}
\mathbb E[\nu_k] & = \sum_{j=1}^n p_{j=1} \left[\frac{1}{np_j}[f'_{j}(x)-f'_{j_k}(\phi_j)]+\frac{1}{n}\sum_{i=1}^{n}f'_i(\phi_i)\right]\\
&  = \sum_{j=1}^n \left[\frac{1}{n}f_{j}'(x) - \frac{1}{n}f_j'(\phi_j) + \frac{p_j}{n}\sum_{i=1}^n f_i'(\phi_i)\right]\\
&  = \frac{1}{n}\sum_{i=1}^n f_{j}'(x) - \frac{1}{n}\sum_{i=1}^n f_{j}'(\phi_j) + \sum_{i=1}^n[p_i]\frac{1}{n}\sum_{i=1}^n f_{j}'(\phi_j)\\
& = \frac{1}{n}\sum_{i=1}^n f_i'(x) = f'(x).
\end{align*}
To show the second part, we use that $\mathbb E\|X - \mathbb E [X] + Y\|^2 = \mathbb E\|X - \mathbb E[X]\|^2 + \mathbb{E}\|Y\|^2$ if $X$ and $Y$ are independent, $\mathbb E\|X-\mathbb E[X]\|^2 \leq \mathbb E\|X\|^2$, and $\|x+y\|^2 \leq 2\|x\|^2+2\|y\|^2 $,
\begin{align*} 
\mathbb E\|\nu_k\|^2
& =\mathbb E\|\frac{1}{np_j}[f'_j(x)-f'_j(\phi_j)]+\frac{1}{n}\sum_{i=1}^{n}f'_i(\phi_i)\|^2\\
& =\mathbb E\|\frac{1}{np_j}[f'_j(x)-f'_j(x^*)]-f'(x)+f'(x)-\frac{1}{np_j}[f'_j(\phi_j)-f'_j(x^*)]-\frac{1}{n}\sum_{i=1}^{n}f'_i(\phi_i))\|^2\\
& =\mathbb E\|\frac{1}{np_j}[f'_j(x)-f'_j(x^*)]-f'(x)-\frac{1}{np_j}[f'_j(\phi_j)-f'_j(x^*)]-\frac{1}{n}\sum_{i=1}^{n}f'_i(\phi_i))\|^2+\|f'(x)\|^2\\
& \leq \mathbb E\|\frac{1}{np_j}[f'_j(x)-f'_j(x^*)]-f'(x)\|^2  + 2\mathbb E\|\frac{1}{np_j}[f'_j(\phi_j)-f'_j(x^*)]-\frac{1}{n}\sum_{i=1}^{n}f'_i(\phi_i))\|^2+\|f'(x)\|^2\\
& \leq 2\mathbb E\|\frac{1}{np_j}[f'_j(x)-f'_j(x^*)]\|^2-2\|f'(x)\|^2+ 2\mathbb E\|\frac{1}{np_j}[f'_j(\phi_j)-f'_j(x^*)]\|^2+\|f'(x)\|^2\\
& \leq 2\mathbb E\|\frac{1}{np_j}[f'_j(x)-f'_j(x^*)]\|^2+2\mathbb E\|\frac{1}{np_j}[f'_j(\phi_j)-f'_j(x^*)]\|^2.
\end{align*}
\end{proof}
\noindent We will also make use of the inequality
\begin{equation}
\label{lemma2}
\langle f'(x),x^*-x\rangle \leq -\frac \mu 2\|x-x^*\|^2 - \frac 1 {2Ln}\sum_{i=1}^n \|f'_i(x^*)-f'_i(x)\|^2,
\end{equation}
which follows from~\citet[Lemma 1]{defazio2014saga} using that $f'(x^*)=0$ and the non-positivity of $\frac{L-\mu}{L}[f(x^*)-f(x)]$. We now give the proof of part (a) of Proposition 1, which we state below.
\setcounter{proposition}{0}
\begin{proposition}[a]
If $\eta = \frac{4L+n\mu}{np_m}$ and $p_m = \min_j\{p_j\}$, then Algorithm~2 has
\[
\mathbb{E}[\norm{x^k-x^*}^2] \leq \left(  1 - \frac{np_m\mu}{n\mu + 4L}  \right)^t \left[\norm{x^0 - x^*} +  \frac{2p_m}{(4L+n\mu)^2}\sum_i\frac{1}{p_i}\norm{\nabla f_i(x^0) - \nabla f_i(x^*)}^2\right].
\]
\end{proposition}
\begin{proof}
 We denote the Lyapunov function $T^k$ at iteration $k$ by
\[
T^k=\frac{1}{n}\sum_{i=1}^n\frac{1}{np_j}\|f'_i(\phi_i^k)-f'_i(x^*)\|^2+c\|x^k-x^*\|^2.
\]
 We will will show that $\mathbb E[T^{k+1}]\leq (1-\frac{1}{\kappa})T^k$ for some $\kappa < 1$. First, we write the expectation of the first term as
\begin{align}
\label{eq:lyapunov1}
& \mathbb E\left[\sum_i \frac 1 {n^2p_i} \|f'_i(\phi_i)-f'_i(x^*)\|^2\right]\nonumber\\
& = \mathbb E\left[\frac 1 {n^2p_j} \|f'_j(x)-f'_j(x^*)\|^2\right]+ \sum_i \frac 1 {n^2p_i} \|f'_i(\phi_i)-f'_i(x^*)\|^2-E\left[\frac 1 {n^2p_j} \|f'_j(\phi_j)-f'_j(x^*)\|^2\right]\nonumber\\
& =\frac 1 {n^2} \sum_i \| f'_i(x)-f'_i(x^*)\|^2+\frac 1 {n^2} \sum_i \left(\frac 1 p_i - 1\right) \| f'_i(\phi_i)-f'_i(x^*)\|^2.
\end{align}
Next, we simplify the other term of $\mathbb E[T^{k+1}]$,
\begin{align*}
c\mathbb E\|x^{k+1}-x^*\|^2 &= c\mathbb E\|x-x^*-\frac 1 \eta \nu_k\|^2\\
& =c\|x-x^*\|^2+\frac c {\eta^2}\mathbb E\|\nu_k\|^2+\frac{2c} \eta \langle f'(x), x-x^*\rangle
\end{align*}
We now use Lemma~\ref{lemma:variance} followed by Inequality~\eqref{lemma2},
\begin{align*}
c\mathbb E\|x^{k+1}-x^*\|^2 & \leq c\|x-x^*\|^2+\frac c {\eta^2} 2\mathbb E\|\frac{1}{np_j}[f'_j(x)-f'_j(x^*)]\|^2+\frac c {\eta^2}2\mathbb E\|\frac{1}{np_j}[f'_j(\phi_j)-f'_j(x^*)]\|^2+\frac{2c} \eta \langle f'(x), x-x^*\rangle\nonumber\\
& \leq c(1-\frac \mu \eta)\|x-x^*\|^2+ \frac {2c} {\eta^2} \mathbb E\|\frac 1 {np_j}(f'_j(x)-f'_j(x^*))\|^2\nonumber\\
& +\frac {2c} {\eta^2} \mathbb E\|\frac 1 {np_j}(f'_j(\phi_j)-f'_j(x^*))\|^2 -\frac c {n\eta L}\sum_i \|f'_i(x^*)- f'_i(x)\|^2\nonumber\\
& = c(1-\frac \mu \eta)\|x-x^*\|^2 + \sum_i ( \frac {2c} {n^2\eta^2p_i}-\frac c {n\eta L})\|f'_i(x)- f'_i(x^*)\|^2+  \sum_i ( \frac {2c} {n^2\eta^2p_i})\|f'_i(\phi_i)- f'_i(x^*)\|^2.
\end{align*}
We use this to bound the expected improvement in the Lyapunov function,
\begin{align*}
\mathbb E[T^{k+1}]-T^k
& = E[T^{k+1}] - \frac{1}{n}\sum_{i=1}^n\frac{1}{np_j}\|f'_i(\phi_i)-f'_i(x^*)\|^2-c\|x-x^*\|^2\\
& \leq \frac 1 {n^2} \sum_i \| f'_i(x)-f'_i(x^*)\|^2+\frac 1 {n^2} \sum_i \left(\frac 1 p_i - 1\right) \| f'_i(\phi_i)-f'_i(x^*)\|^2 & \text{From \eqref{eq:lyapunov1}} \\
& + c(1-\frac \mu \eta)\|x-x^*\|^2 + \sum_i ( \frac {2c} {n^2\eta^2p_i}-\frac c {n\eta L})\|f'_i(x)- f'_i(x^*)\|^2 +  \sum_i ( \frac {2c} {n^2\eta^2p_i})\|f'_i(\phi_i)- f'_i(x^*)\|^2 & \text{From above}\\
&  - \frac{1}{n}\sum_{i=1}^n\frac{1}{np_j}\|f'_i(\phi_i)-f'_i(x^*)\|^2-c\|x-x^*\|^2 & \text{Def'n of $T^k$}\\
& = \frac 1 {n^2} \sum_i \| f'_i(x)-f'_i(x^*)\|^2-\frac 1 {n^2} \sum_i \| f'_i(\phi_i)-f'_i(x^*)\|^2 \\
& -\frac{c \mu}{ \eta}\|x-x^*\|^2 + \sum_i ( \frac {2c} {n^2\eta^2p_i}-\frac c {n\eta L})\|f'_i(x)- f'_i(x^*)\|^2 +  \sum_i ( \frac {2c} {n^2\eta^2p_i})\|f'_i(\phi_i)- f'_i(x^*)\|^2\\
& = -\frac 1 \kappa T^k + \left( \frac 1 \kappa - \frac \mu \eta \right)c\|x-x*\|^2 & (*) \\
& +\sum_i \left( \frac {2c} {n^2\eta^2p_i}+\frac 1 {n^2}-\frac c {n\eta L}\right)\|f'_i(x)- f'_i(x^*)\|^2\\
& + \sum_i \left( \frac {2c} {n^2\eta^2p_i}-\frac 1 {n^2}+\frac 1 {n^2\kappa p_i}\right)\|f'_i(\phi_i)- f'_i(x^*)\|^2\\
& \leq  -\frac 1 \kappa T^k + \left( \frac 1 \kappa - \frac \mu \eta \right)\left[c\|x-x*\|^2\right]\\
& +\left( \frac {2c} {n^2\eta^2p_m}+\frac 1 {n^2}-\frac c {n\eta L}\right)\left[\sum_i \|f'_i(x)- f'_i(x^*)\|^2\right]\\
& + \left( \frac {2c} {n^2\eta^2p_m}-\frac 1 {n^2}+\frac 1 {n^2\kappa p_m}\right)\left[\sum_i \|f'_i(\phi_i)- f'_i(x^*)\|^2\right],\\
\end{align*}
where in $(*)$ we add and subtract $\frac{1}{\kappa}T^k$ and in the last line we assumed $c \geq 0$ and used $p_i \geq p_m$. The terms in square brackets are positive, and if we can choose the constants $\{c,\kappa,\eta\}$ to make the round brackets non-positive, we have
\[
\mathbb E[T^{k+1}] \leq \left(1-\frac 1 \kappa\right) T^k.
\]
For the first expression, choosing $ \kappa = \frac  \eta \mu$ makes it zero. We can make the third expression zero under this choice of $\kappa$ by choosing $ c = \frac {\eta^2p_m} 2 - \frac {\mu \eta} 2$. This follows because  
\[
 \frac {2c} {n^2\eta^2p_m}-\frac 1 {n^2}+\frac 1 {n^2\kappa p_m}= \frac {2c} {n^2\eta^2p_m}-\frac 1 {n^2}+\frac \mu {n^2\eta p_m} = 0,
\]
is equivalent to
\[
 \frac {2c} {n^2\eta^2p_m}= \frac 1 {n^2}-\frac \mu {n^2\eta p_m} \Leftrightarrow
  c = \frac {\eta^2p_m} 2 - \frac {\mu \eta} 2.
\]
For the second expression, note that with our choice of $c$ we have
\[
 \frac {2c} {n^2\eta^2p_m}+\frac 1 {n^2}-\frac c {n\eta L} = \frac 1 {n^2}-\frac \mu {n^2\eta p_m}+\frac 1 {n^2}-\frac  {\frac {\eta^2p_m} 2 - \frac {\mu \eta} 2}  {n\eta L}, 
\]
which (multiplying by $n$) is negative if we have
\[
  \frac 2 n + \frac \mu {2L} \leq \frac \mu {n\eta p_m} + \frac {\eta p_m} {2L}.
\]
 Ignoring the last term, we can choose 
\[
\eta = \frac {4L+n\mu} {np_m}.
\]
We will also require that $c \geq 0$ to complete the proof, but this follows because $\eta \geq \frac{\mu}{p_m}$. By using that
\[
c\mathbb{E}[\norm{x^{k+1}-x^*}^2] \leq \mathbb{E}[T^{k+1}] \leq \left(1-\frac{1}{\kappa}\right)T^k = \left(1-\frac{\mu}{\eta}\right)T^k
\]
and chaining the expectations while using the definition of $\eta$ we obtain
\begin{align*}
\mathbb{E}[\norm{x^k - x^*}^2] & \leq \left(1 - \frac{\mu}{\eta}\right)^k\frac{T^0}{c}\\
& = \left(1 - \frac{np_m\mu}{n\mu + 4L}\right)^k\left[\|x^0-x^*\|^2 + \frac{1}{cn}\sum_{i=1}^n\frac{1}{np_j}\|f'_i(\phi_i^0)-f'_i(x^*)\|^2\right].
\end{align*}
To get the final expression, use that
\begin{align*}
\frac{1}{cn^2} = \frac{2}{n^2(\eta^2p_m - \mu \eta)} \leq \frac{2}{n^2\eta^2p_m} = \frac{2n^2p_m^2}{n^2p_m(4L+n\mu)^2} = \frac{2p_m}{(4L+n\mu)^2}.
\end{align*}
\end{proof}

\section*{Appendix B: Proof of Part (b) of Proposition 1}

In this section we consider the minimization problem
\[
\min_x f(x) = \frac{1}{n}\sum_{i=1}^nf_i(x),
\]
where each $f_i'$ is $L_i$-Lipschitz continuous and $f$ is $\mu$-strongly-convex. We will define Algorithm~3 by the sequences $\{x^k\}$, $\{\nu_k\}$, and $\{\phi_j^k\}$ given by
\begin{align*}
    \nu_k & =\frac{\bar{L}}{L_i}[f'_{j_k}(x^k)-f'_{j_k}(\phi_j^k)]+\frac{1}{n}\sum_{i=1}^{n}f'_i(\phi_i^k),\\
    x^{k+1} &=x^k-\gamma\nu_k,\\
    \phi_{j}^{k+1}& =\begin{cases}f_{r_k}'(x^k) & \text{if  $j = r_k$,}\\
\phi_j^k & \text{otherwise,}\end{cases}
\end{align*}
where $j_k = j$ with probability $\frac{L_i}{\sum_{j=1}^nL_j}$ and $r_k$ is picked uniformly at random. This is identical to Algorithm~2, except it uses a specific choice of the $p_j$ and the memory $\phi_j$ is updated based on a different random sample that is sampled uniformly. This algorithm maintains the key property that the expected step is a gradient step, $\mathbb E[\nu_k]=f'(x^k)$.

\noindent From our assumptions about $f$ and the $f_i$, we have~\citep[][see Chapter 2]{nesterov2004introductory}.
\begin{equation}
\label{eq:lipFi}
f_i(x)\geq f_i(y)+\left\langle f_i^{\prime}(y),x-y\right\rangle +\frac{1}{2L}\left\Vert f_i^{\prime}(x)-f_i^{\prime}(y)\right\Vert ^{2},
\end{equation}
and
\begin{equation}
\label{eq:sc}
f(x)\geq f(y)+\left\langle f^{\prime}(y),x-y\right\rangle +\frac{\mu}{2}\left\Vert x-y\right\Vert ^{2}.
\end{equation}
We use these to derive several useful inequalities that we will use in the analysis.
Adding the former times $\frac{1}{2n}$ for all $i$ to the latter times $\frac{1}{2}$ for $y=x^*$ gives the inequality
\begin{equation}
\label{cor:saga-ip-bound} 
\left\langle f^{\prime}(x),x^{*}-x\right\rangle  \leq f(x^{*})-f(x)-\frac{\mu}{4}\left\Vert x^{*}-x\right\Vert ^{2}-\frac{1}{4n}\sum_{i}\frac{1}{L_{i}}\left\Vert f_{i}^{\prime}(x^{*})-f_{i}^{\prime}(x)\right\Vert ^{2}.
\end{equation}
Also by applying~\eqref{eq:lipFi}  with $y=x^*$ and $x=\phi_i$, for each $f_{i}$ and summing,
we have that for all $\phi_{i}$ and $x^{*}$:
\begin{equation}
\label{cor:grad-diff-phii}
\frac{1}{n}\sum_{i}\frac{1}{L_{i}}\left\Vert f_{i}^{\prime}(\phi_{i})-f_{i}^{\prime}(x^{*})\right\Vert ^{2}\leq\frac{2}{n}\sum_{i}\left[f_{i}(\phi_{i})-f(x^{*})-\left\langle f_{i}^{\prime}(x^{*}),\phi_{i}-x^{*}\right\rangle \right].
\end{equation}
Further, by both minimizing sides of~\eqref{eq:sc} we obtain
\begin{equation}
\label{eq:scGrad}
-\left\Vert f^{\prime}(x)\right\Vert ^{2}\leq-2\mu\left[f(x)-f(x^{*})\right].
\end{equation}
We next derive a bound on the variance of the gradient estimate.
\begin{lemma}
\label{lem:saga-error-bound}It holds that for any $\phi_{i}$ that
with $x^{k+1}$ and $x^{k}$ as given by Algorithm~2 we have
\begin{eqnarray*}
\mathbb{E}\left\Vert x^{k+1}\!-\! x^{k}\right\Vert ^{2} & \!\leq\! & 2\gamma^{2}\frac{\bar{L}}{n}\sum_{i}\frac{1}{L_{i}}\left\Vert f_{j}^{\prime}(\phi_{j}^{k})-f_{j}^{\prime}(x^{*})\right\Vert ^{2}\\
 & \!\! & +2\gamma^{2}\frac{\bar{L}}{n}\sum_{i}\frac{1}{L_{i}}\left\Vert f_{j}^{\prime}(x^{k})-f_{j}^{\prime}(x^{*})\right\Vert ^{2}-\stepsize^{2}\left\Vert f^{\prime}(x^{k})\right\Vert ^{2}.
\end{eqnarray*}
\end{lemma}

\begin{proof}
We again follow the SAGA argument closely here

\begin{eqnarray*}
 &  & \mathbb{E}\left\Vert x^{k+1}-x^{k}\right\Vert ^{2}\\
 & = & \gamma^{2}\mathbb{E}\left\Vert \frac{\bar{L}}{L_{j}}\left[f_{j}^{\prime}(\phi_{j}^{k})-f_{j}^{\prime}(x^{k})\right]-\frac{1}{n}\sum_{i=1}^{n}f_{i}^{\prime}(\phi_{i}^{k})\right\Vert ^{2}\\
 & = & \gamma^{2}\mathbb{E}\left\Vert \frac{\bar{L}}{L_{j}}\left[f_{j}^{\prime}(\phi_{j}^{k})-f_{j}^{\prime}(x^{*})\right]-\frac{1}{n}\sum_{i=1}^{n}f_{i}^{\prime}(\phi_{i}^{k})-\frac{\bar{L}}{L_{j}}\left[f_{j}^{\prime}(x^{k})-f_{j}^{\prime}(x^{*})\right]-f^{\prime}(x^{k})\right\Vert ^{2}\\
 &  & +\stepsize^{2}\left\Vert f^{\prime}(x^{k})\right\Vert ^{2}\\
 & \leq & 2\gamma^{2}\mathbb{E}\left\Vert \frac{\bar{L}}{L_{j}}\left[f_{j}^{\prime}(\phi_{j}^{k})-f_{j}^{\prime}(x^{*})\right]-\frac{1}{n}\sum_{i=1}^{n}f_{i}^{\prime}(\phi_{i}^{k})\right\Vert ^{2}\\
 &  & +2\gamma^{2}\mathbb{E}\left\Vert \frac{\bar{L}}{L_{j}}\left[f_{j}^{\prime}(x^{k})-f_{j}^{\prime}(x^{*})\right]-f^{\prime}(x^{k})\right\Vert ^{2}+\stepsize^{2}\left\Vert f^{\prime}(x^{k})\right\Vert ^{2}\\
 & \leq & 2\gamma^{2}\mathbb{E}\left\Vert \frac{\bar{L}}{L_{j}}\left[f_{j}^{\prime}(\phi_{j}^{k})-f_{j}^{\prime}(x^{*})\right]\right\Vert ^{2}\\
 &  & +2\gamma^{2}\mathbb{E}\left\Vert \frac{\bar{L}}{L_{j}}\left[f_{j}^{\prime}(x^{k})-f_{j}^{\prime}(x^{*})\right]\right\Vert ^{2}-\stepsize^{2}\left\Vert f^{\prime}(x^{k})\right\Vert ^{2}.
\end{eqnarray*}
We can expand those expectations as follows:
\begin{eqnarray*}
\mathbb{E}\left\Vert \frac{\bar{L}}{L_{i}}\left[f_{j}^{\prime}(\phi_{j}^{k})-f_{j}^{\prime}(x^{*})\right]\right\Vert ^{2} & = & \frac{1}{n\bar{L}}\sum_{i}L_{i}\left\Vert \frac{\bar{L}}{L_{i}}\left[f_{j}^{\prime}(\phi_{j}^{k})-f_{j}^{\prime}(x^{*})\right]\right\Vert ^{2}\\
 & = & \frac{\bar{L}}{n}\sum_{i}\frac{1}{L_{i}}\left\Vert \left[f_{j}^{\prime}(\phi_{j}^{k})-f_{j}^{\prime}(x^{*})\right]\right\Vert ^{2},
\end{eqnarray*}
and similarly for $\mathbb{E}\left\Vert \frac{\bar{L}}{L_{i}}\left[f_{j}^{\prime}(x^{k})-f_{j}^{\prime}(x^{*})\right]\right\Vert ^{2}.$
\end{proof}
\noindent We now give the proof of part (b) of Proposition 1, which we state below.
\setcounter{proposition}{0}
\begin{proposition}[b]
\label{thm:saga-main-lyp}
If $\gamma = \frac{1}{4L}$, then Algorithm~3 has
\[
E\left[\left\Vert x^{k}-x^{*}\right\Vert ^{2}\right]\leq\left(1-\min\left\{ \frac{1}{3n},\frac{\mu}{8\bar{L}}\right\} \right)^{k}\left[\left\Vert x^{k}-x^{*}\right\Vert ^{2}+\frac{n}{2\bar{L}}\left(f(x^{0})-f(x^{*})\right)\right].
\]
\end{proposition}
\begin{proof}
We
define the Lyapunov function as
\[
T^{k}=\frac{1}{n}\sum_{i}f_{i}(\phi_{i}^{k})-f(x^{*})-\frac{1}{n}\sum_{i}\left\langle f_{i}^{\prime}(x^{*}),\phi_{i}^{k}-x^{*}\right\rangle +c\left\Vert x^{k}-x^{*}\right\Vert ^{2}.
\]
The expectations of the first terms in $T^{k+1}$ are straightforward to simplify:
\begin{eqnarray*}
\mathbb{E}\left[\frac{1}{n}\sum_{i}f_{i}(\phi_{i}^{k+1})\right] & = & \frac{1}{n}f(x^{k})+\left(1-\frac{1}{n}\right)\frac{1}{n}\sum_{i}f_{i}(\phi_{i}^{k}),\\
\mathbb{E}\left[-\frac{1}{n}\sum_{i}\left\langle f_{i}^{\prime}(x^{*}),\phi_{i}^{k+1}-x^{*}\right\rangle \right] & = & -\left(1-\frac{1}{n}\right)\frac{1}{n}\sum_{i}\left\langle f_{i}^{\prime}(x^{*}),\phi_{i}^{k}-x^{*}\right\rangle .
\end{eqnarray*}
Note that these terms make use of the uniformly sampled $\phi_{r}^{k+1}=x^{k}$
value. 
For the change in the last term of $T^k$ we expand the quadratic and
apply $\mathbb{E}[x^{k+1}]=x^{k}-\stepsize f^{\prime}(x^{k})$ to
simplify the inner product term:
\begin{align*}
 & c\mathbb{E}\left\Vert x^{k+1}-x^{*}\right\Vert ^{2}\\
= & c\mathbb{E}\left\Vert x^{k}-x^{*}+x^{k+1}-x^{k}\right\Vert ^{2}\\
= & c\left\Vert x^{k}-x^{*}\right\Vert ^{2}+2c\mathbb{E}\left[\left\langle x^{k+1}-x^{k},x^{k}-x^{*}\right\rangle \right]+c\mathbb{E}\left\Vert x^{k+1}-x^{k}\right\Vert ^{2}\\
= & c\left\Vert x^{k}-x^{*}\right\Vert ^{2}-2c\stepsize\left\langle f^{\prime}(x^{k}),x^{k}-x^{*}\right\rangle +c\mathbb{E}\left\Vert x^{k+1}-x^{k}\right\Vert ^{2}.
\end{align*}
We now apply Lemma~\ref{lem:saga-error-bound} to bound the error
term $c\mathbb{E}\left\Vert x^{k+1}-x^{k}\right\Vert ^{2}$, giving:
\begin{eqnarray*}
 &  & c\mathbb{E}\left\Vert x^{k+1}-x^{*}\right\Vert ^{2}\\
 & \leq & c\left\Vert x^{k}-x^{*}\right\Vert ^{2}-c\stepsize^{2}\left\Vert f^{\prime}(x^{k})\right\Vert ^{2}\\
 &  & -2c\stepsize\left\langle f^{\prime}(x^{k}),x^{k}-x^{*}\right\rangle \\
 &  & +2c\gamma^{2}\frac{\bar{L}}{n}\sum_{i}\frac{1}{L_{i}}\left\Vert f_{i}^{\prime}(\phi_{i}^{k})-f_{i}^{\prime}(x^{*})\right\Vert ^{2}+2c\gamma^{2}\frac{\bar{L}}{n}\sum_{i}\frac{1}{L_{i}}\left\Vert f_{i}^{\prime}(x^{k})-f_{i}^{\prime}(x^{*})\right\Vert ^{2}.
\end{eqnarray*}
Now we bound $-2c\gamma\left\langle f^{\prime}(x),x-x^{*}\right\rangle $
with~\eqref{cor:saga-ip-bound} and then apply~\eqref{cor:grad-diff-phii}
to bound $\mathbb{E}\left\Vert f_{j}^{\prime}(\phi_{j})-f_{j}^{\prime}(x^{*})\right\Vert ^{2}$:
\begin{align*}
c\mathbb{E}\left\Vert x^{k+1}-x^{*}\right\Vert ^{2}\leq & \left(c-\frac{1}{2}c\stepsize\mu\right)\left\Vert x^{k}-x^{*}\right\Vert ^{2}\\
 & +\left(2c\stepsize^{2}\bar{L}-\frac{1}{2}c\gamma\right)\frac{1}{n}\sum_{i}\frac{1}{L_{i}}\left\Vert f_{i}^{\prime}(x^{k})-f_{i}^{\prime}(x^{*})\right\Vert ^{2}\!-\! c\stepsize^{2}\left\Vert f^{\prime}(x^{k})\right\Vert ^{2}\\
 & -2c\stepsize\left[f(x^{k})-f(x^{*})\right]\\
 & +\left(4c\stepsize^{2}\bar{L}\right)\frac{1}{n}\sum_{i}\left[f_{i}(\phi_{i})-f_{i}(x^{*})-\left\langle f_{i}^{\prime}(x^{*}),\phi_{i}-x^{*}\right\rangle \right].
\end{align*}
We can now combine the bounds we have derived for each term in $T$,
and pull out a fraction $\frac{1}{\kappa}$ of $T^{k}$ (for any $\kappa$
at this point). Together with~\eqref{eq:scGrad} this yields:{\small{}
\begin{align}
\mathbb{E}[T^{k+1}]-T^{k}\leq & -\frac{1}{\kappa}T^{k}+\left(\frac{1}{n}-2c\stepsize-2c\stepsize^{2}\mu\right)\left[f(x^{k})-f(x^{*})\right]\nonumber \\
 & +\left(\frac{1}{\kappa}+4c\stepsize^{2}\bar{L}-\frac{1}{n}\right)\left[\frac{1}{n}\sum_{i}f_{i}(\phi_{i}^{k})-f(x^{*})-\frac{1}{n}\sum_{i}\left\langle f_{i}^{\prime}(x^{*}),\phi_{i}^{k}-x^{*}\right\rangle \right]\nonumber \\
 & +\left(\frac{1}{\kappa}-\frac{1}{2}\stepsize\mu\right)c\left\Vert x^{k}-x^{*}\right\Vert ^{2}+\left(2\stepsize\bar{L}-\frac{1}{2}\right)c\stepsize\frac{1}{n}\sum_{i}\frac{1}{L_{i}}\left\Vert f_{i}^{\prime}(x^{k})-f_{i}^{\prime}(x^{*})\right\Vert ^{2}.\label{eq:constants-strong}
\end{align}
}{\small \par}
Note that the term in square brackets in the second row is positive in light of~\eqref{cor:grad-diff-phii}.
We now attempt to find constants
that satisfy the required relations. We start with naming
the constants that we need to be non-positive:
\[
c_{1}=\frac{1}{n}-2c\stepsize-2c\stepsize^{2}\mu,
\]
\[
c_{2}=\frac{1}{\kappa}+4c\stepsize^{2}\bar{L}-\frac{1}{n},
\]
\[
c_{3}=\frac{1}{\kappa}-\frac{1}{2}\stepsize\mu,
\]
\[
c_{4}=2\stepsize\bar{L}-\frac{1}{2}.
\]
Recall that we are using the step size $\gamma=1/4\bar{L}$, and thus $c_4=0$.
Setting $c_{1}$ to zero gives
\[
c=\frac{1}{2\gamma(1-\gamma\mu)n},
\]
which is positive since $\gamma\mu < 1$.
Now we look at the restriction that $c_{2}\leq0$ places on $\kappa$:
\begin{eqnarray*}
 \frac{1}{\kappa}+4c\stepsize^{2}\bar{L}-\frac{1}{n}
 & = & \frac{1}{\kappa}+\frac{4\gamma\bar{L}}{2(1-\gamma\mu)n}-\frac{1}{n}\\
 & = & \frac{1}{\kappa}+\frac{1}{2(1-\gamma\mu)n}-\frac{1}{n}\\
 & = & \frac{1}{\kappa}+\frac{1}{2(1-\mu/4\bar{L})n}-\frac{1}{n}\\
 & \leq & \frac{1}{\kappa}+\frac{1}{2(1-\bar{L}/4\bar{L})n}-\frac{1}{n}\\
 & = & \frac{1}{\kappa}+\frac{2}{3n}-\frac{1}{n}\\
 & = & \frac{1}{\kappa}-\frac{1}{3n},
\end{eqnarray*}
\[
\therefore\frac{1}{\kappa}\leq\frac{1}{3n}.
\]
We also have the restriction from $c_{3}=\frac{1}{\kappa}-\frac{1}{2}\stepsize\mu$
of
\[
\frac{1}{\kappa}\leq\frac{\mu}{8\bar{L}},
\]
therefore we can take
\[
\frac{1}{\kappa}=\min\left\{ \frac{1}{3n},\frac{\mu}{8\bar{L}}\right\} .
\]
Note that $c\left\Vert x^{k}-x^{*}\right\Vert ^{2}\leq T^{k}$, and
therefore by chaining expectations and plugging in constants we get:
\[
E\left[\left\Vert x^{k}-x^{*}\right\Vert ^{2}\right]\leq\left(1-\min\left\{ \frac{1}{3n},\frac{\mu}{8\bar{L}}\right\} \right)^{k}\left[\left\Vert x^{k}-x^{*}\right\Vert ^{2}+\frac{n}{2\bar{L}}\left(f(x^{0})-f(x^{*})\right)\right].
\]
\end{proof}

\section*{Appendix C: Test Error Plots for All Methods}

\begin{figure*}
\begin{center}
 \fig{.35}{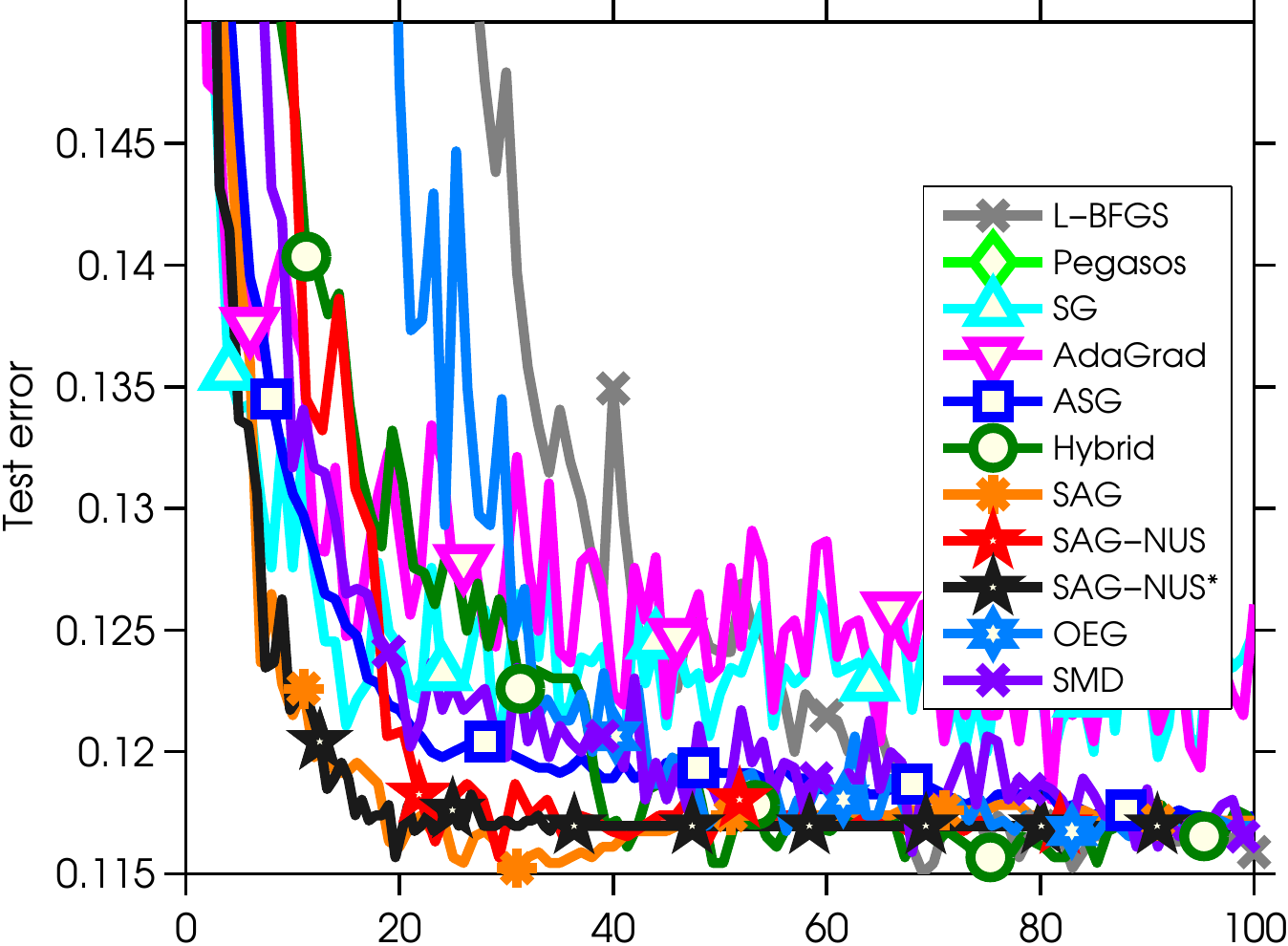}
\fig{.35}{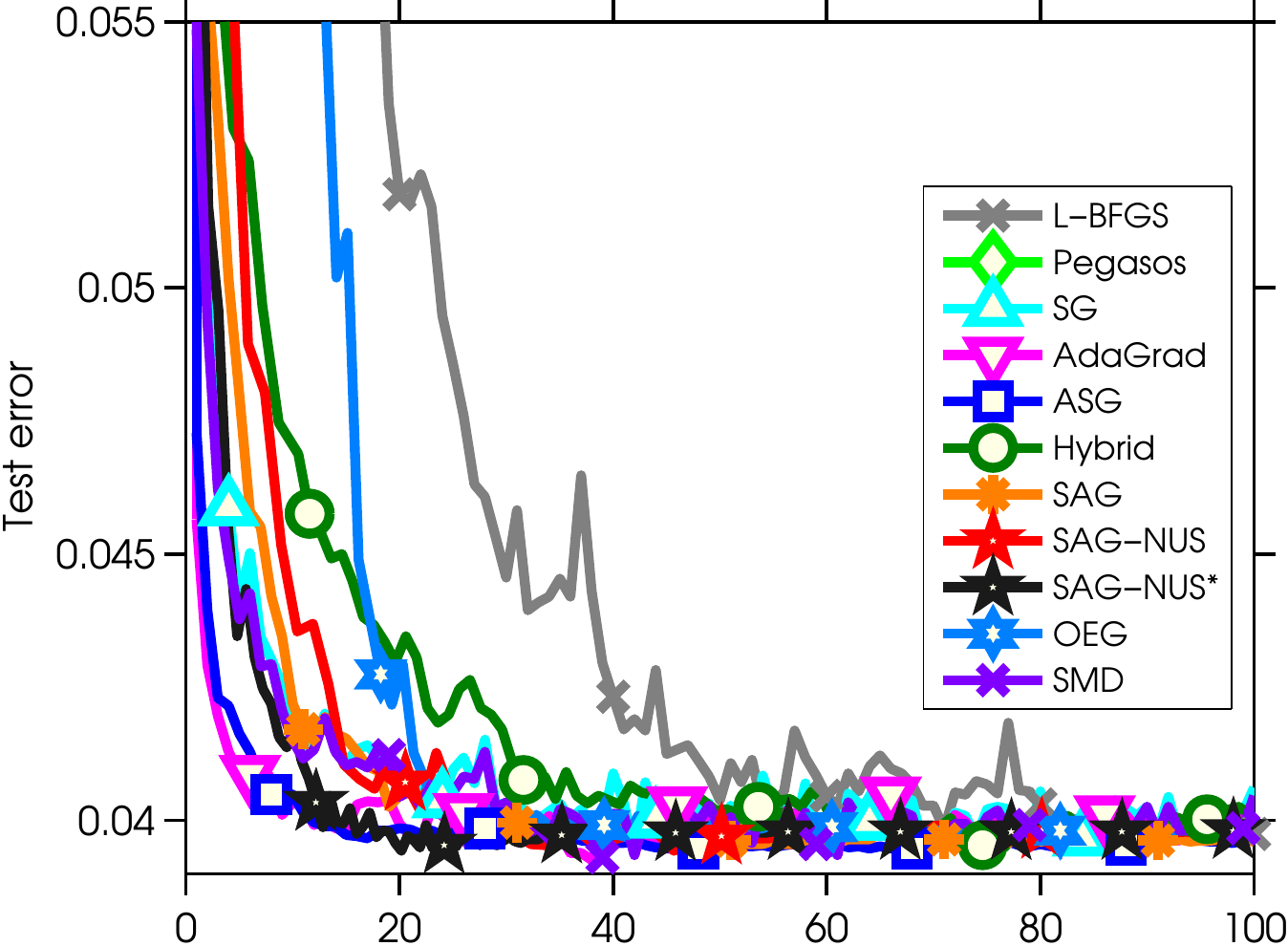}\\
\fig{.35}{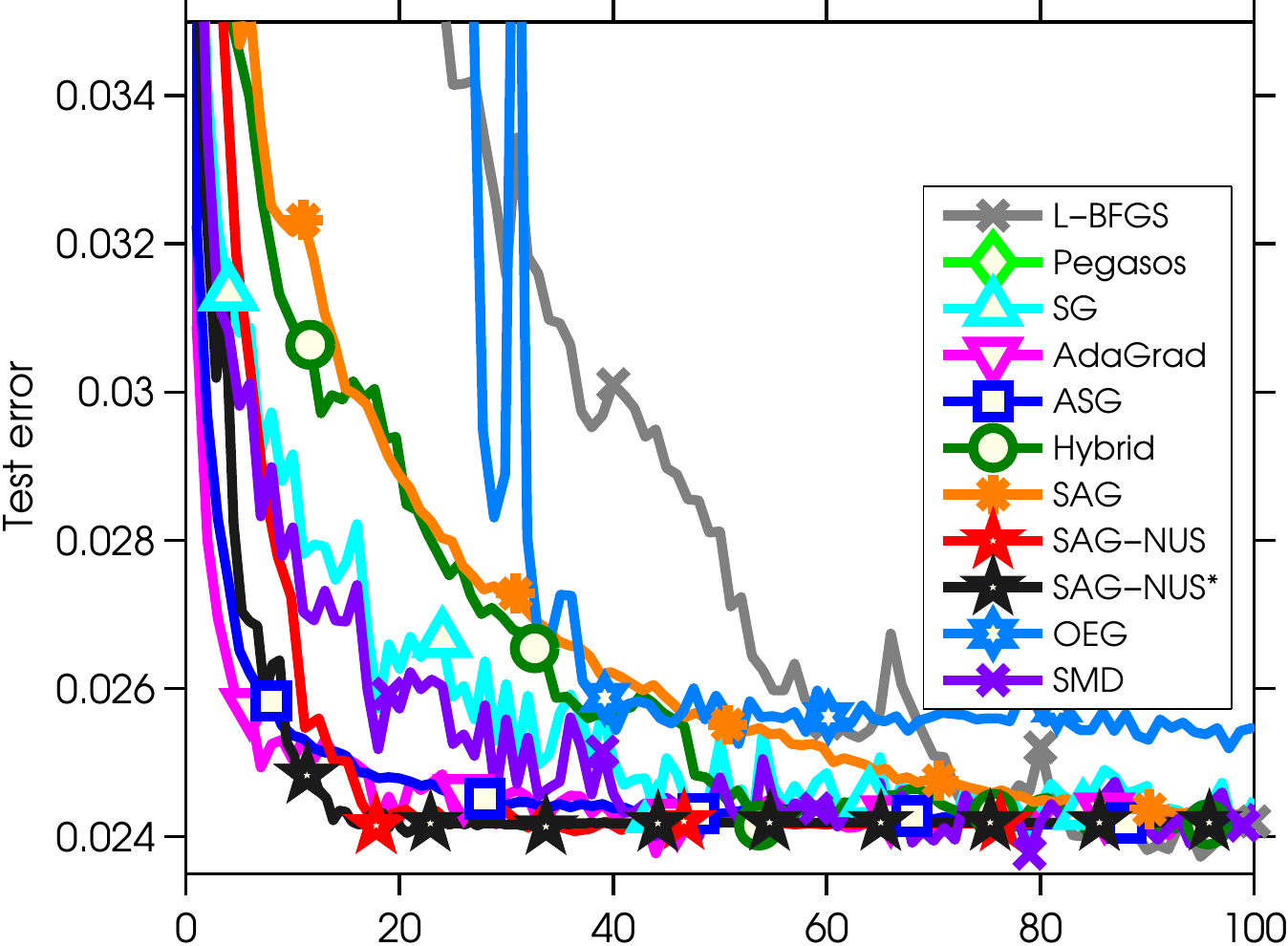}
\fig{.35}{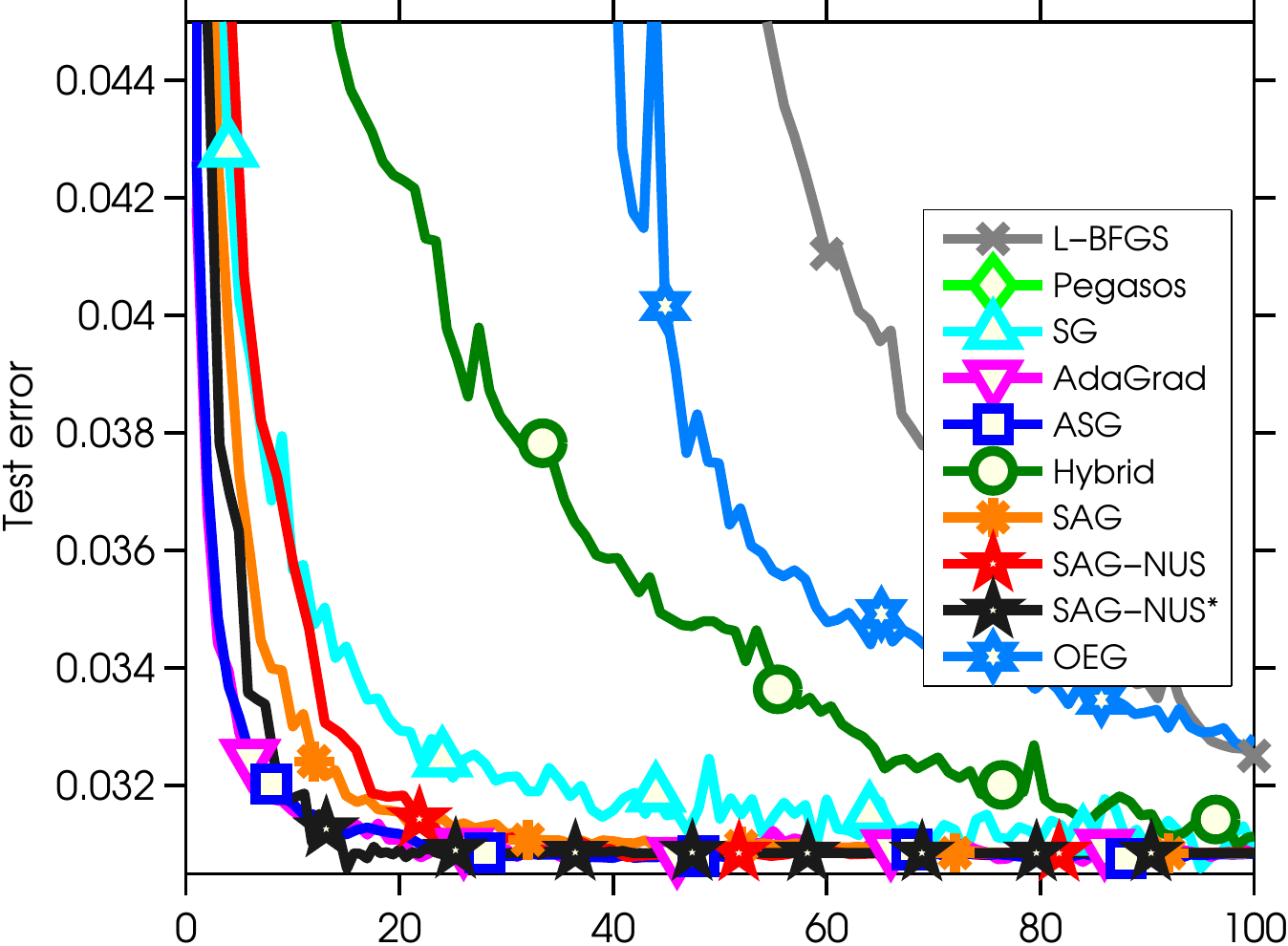}
\end{center}
\caption{Test error against effective number of passes for different deterministic, stochastic, and semi-stochastic optimization strategies (this figure is best viewed in colour). Top-left: OCR, Top-right: CoNLL-2000, bottom-left: CoNLL-2002, bottom-right: POS-WSJ.}
\label{fig:test}
\end{figure*}

In the main body we only plotted test error for a subset of the methods. In Figure~\ref{fig:test} we plot the test error of all methods considered in Figure 1. Note that Pegasos does not appear on the plot (despite being in the legend) because its values exceed the maximum plotted values. In these plots we see that the SAG-NUS methods perform similarly to the best among the optimally-tuned stochastic gradient methods in terms of test error, despite the lack of tuning required to apply these methods.

\section*{Appendix D: Improved Results for OEG}

\begin{figure}
\begin{center}
 \fig{.35}{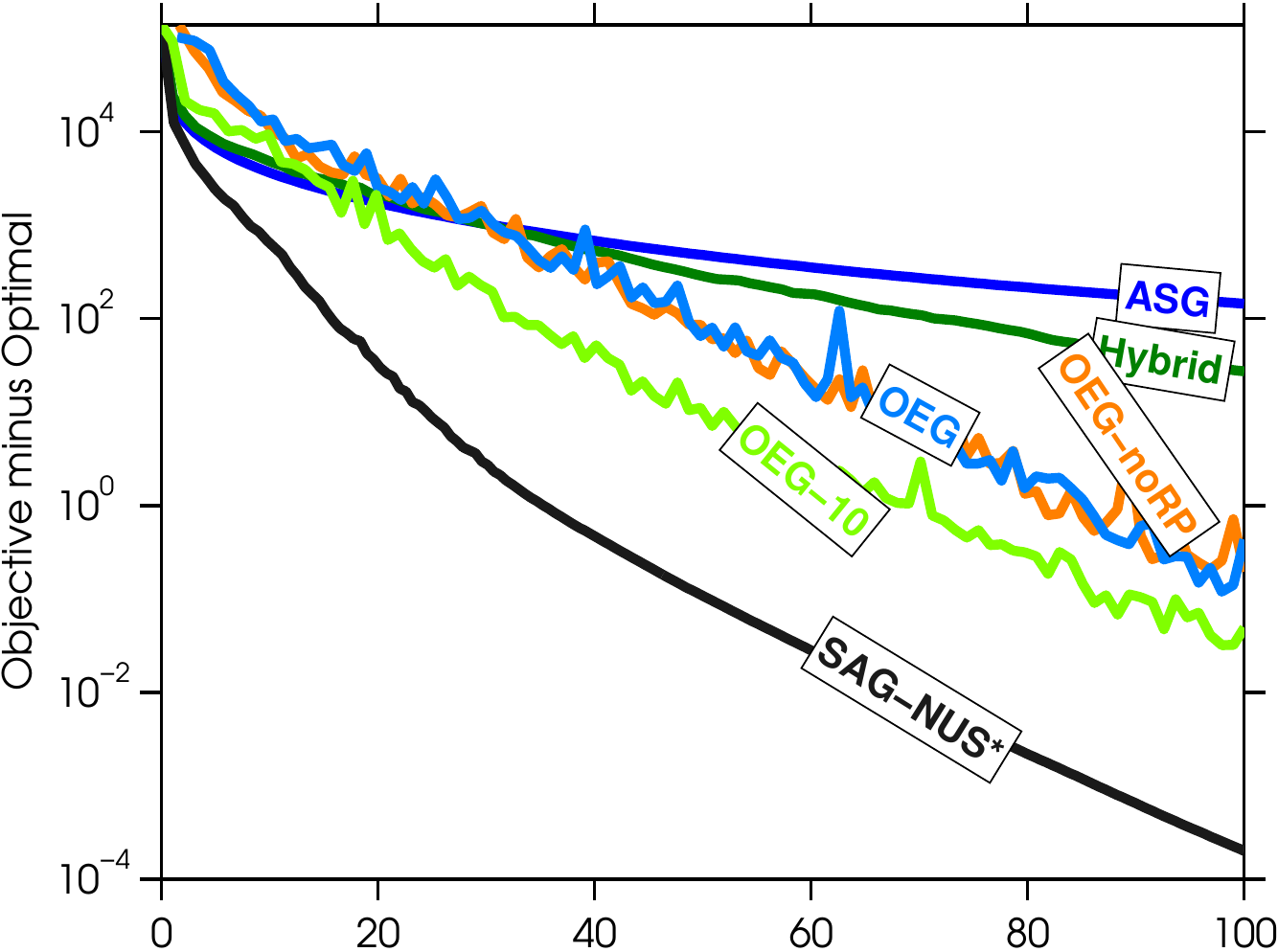}
\fig{.35}{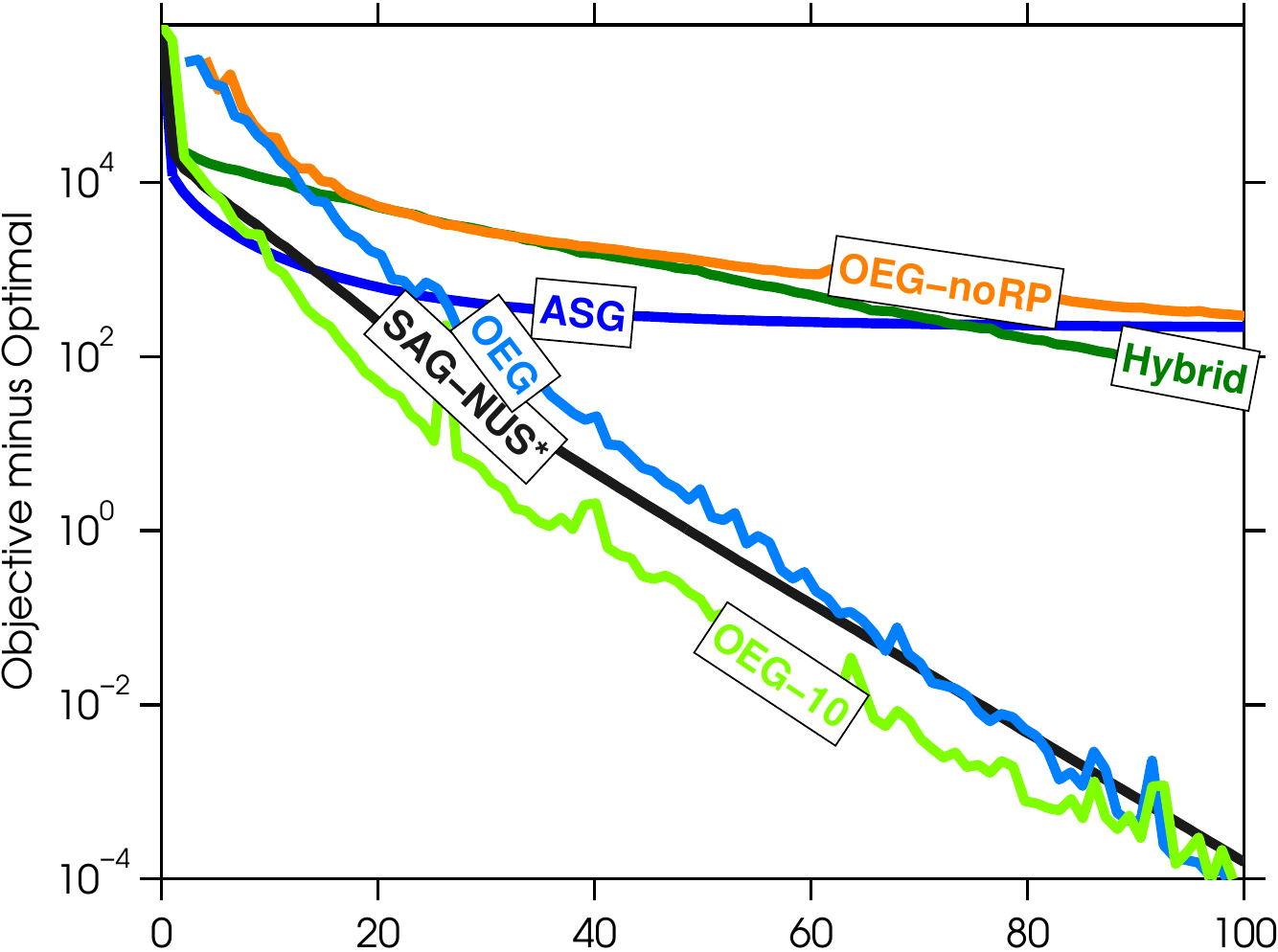}\\
\fig{.35}{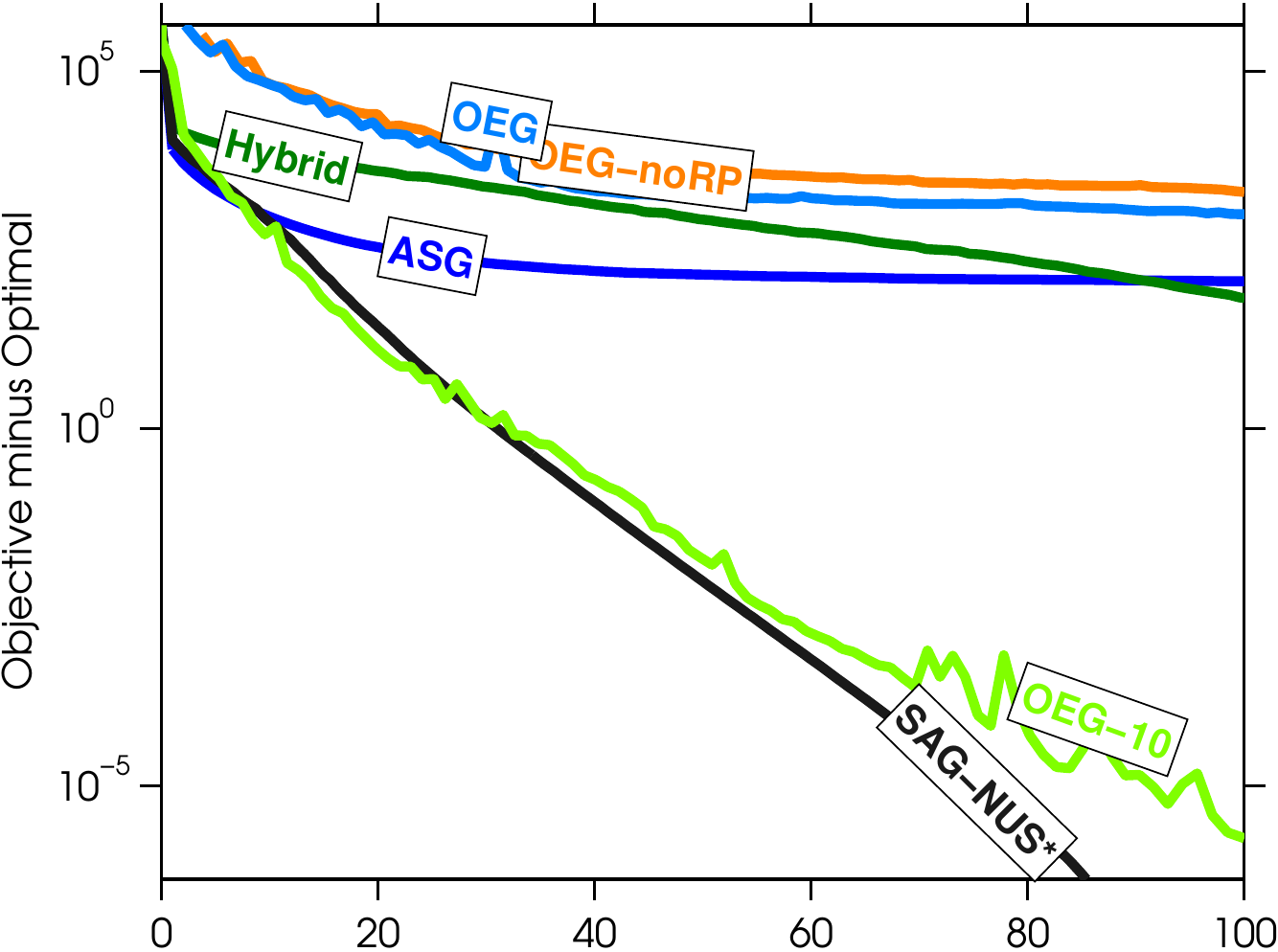}
\fig{.35}{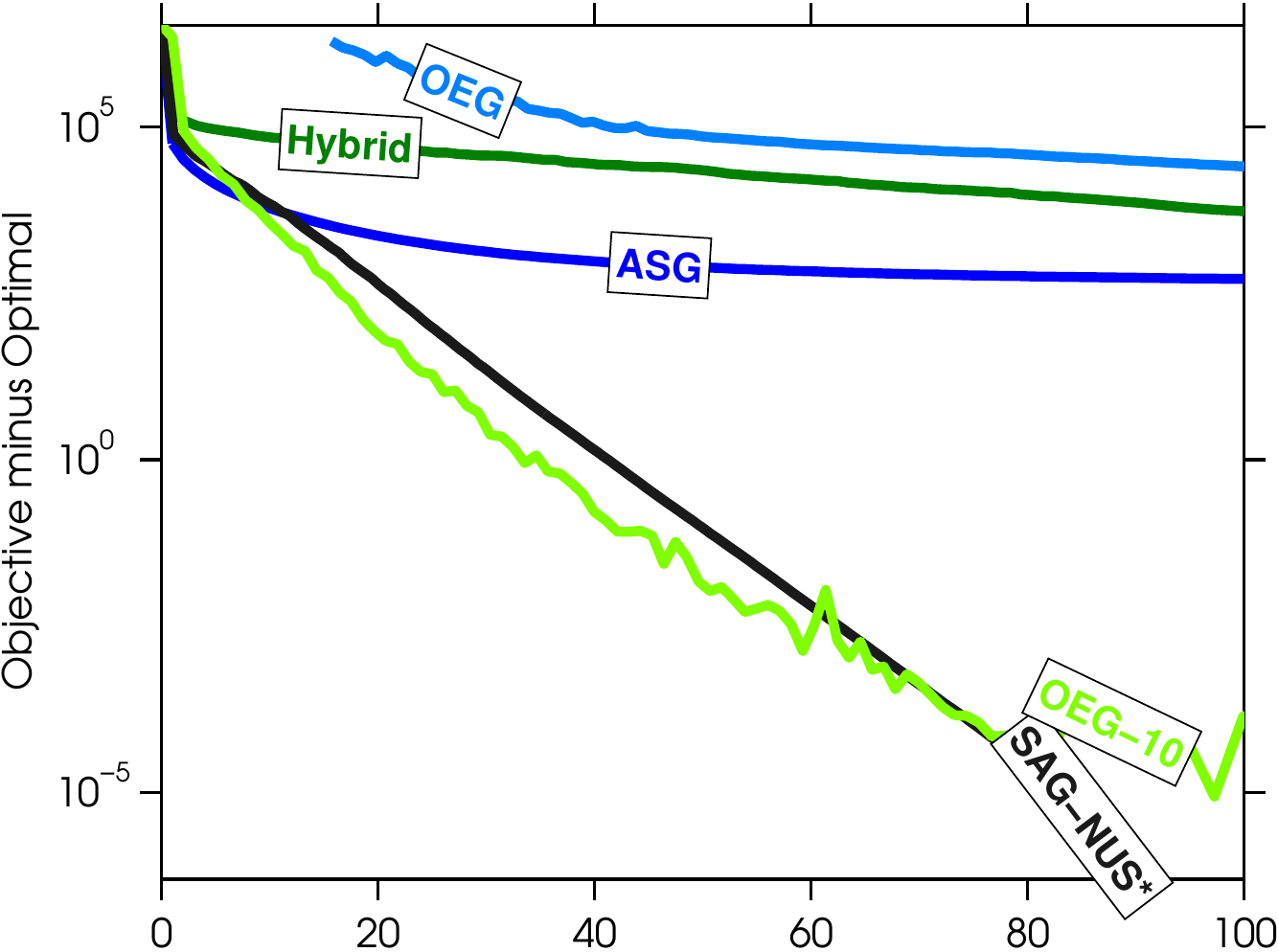} 
\end{center}
\caption{Objective minus optimal objective value against effective number of passes for different variants of OEG. Top-left: OCR, Top-right: CoNLL-2000, bottom-left: CoNLL-2002, bottom-right: POS-WSJ.}
\label{fig:trainOEG}
\end{figure}


Owing to the high variance of the performance of the OEG method, we explored whether better performance could be obtained with the OEG method. The two most salient observations from these experiments where that (i) utilizing a random permutation on the first pass through the data seems to be crucial to performance, and (ii) that better performance could be obtained on the two datasets where OEG performed poorly by using a different initialization. In particular, better performance could be obtained by initializing the parts with the correct labels to a larger value, such as 10. In Figure~\ref{fig:trainOEG}, we plot the performance of the OEG method without using the random permutation (\emph{OEG-noRP}) as well as OEG with this initialization (\emph{OEG-10}). Removing the random permutation makes OEG perform much worse on one of the datasets, while using the different initialization makes OEG perform nearly as well as SAG-NUS* on the datasets where previously it performed poorly (although it does not make up the performance gap on the remaining data set). Performance did not further improve by using even larger values in the initialization, and using a value that was too large lead to numerical problems.

\section*{Appendix E: Runtime Plots}

\begin{figure*}
\begin{center}
 \fig{.35}{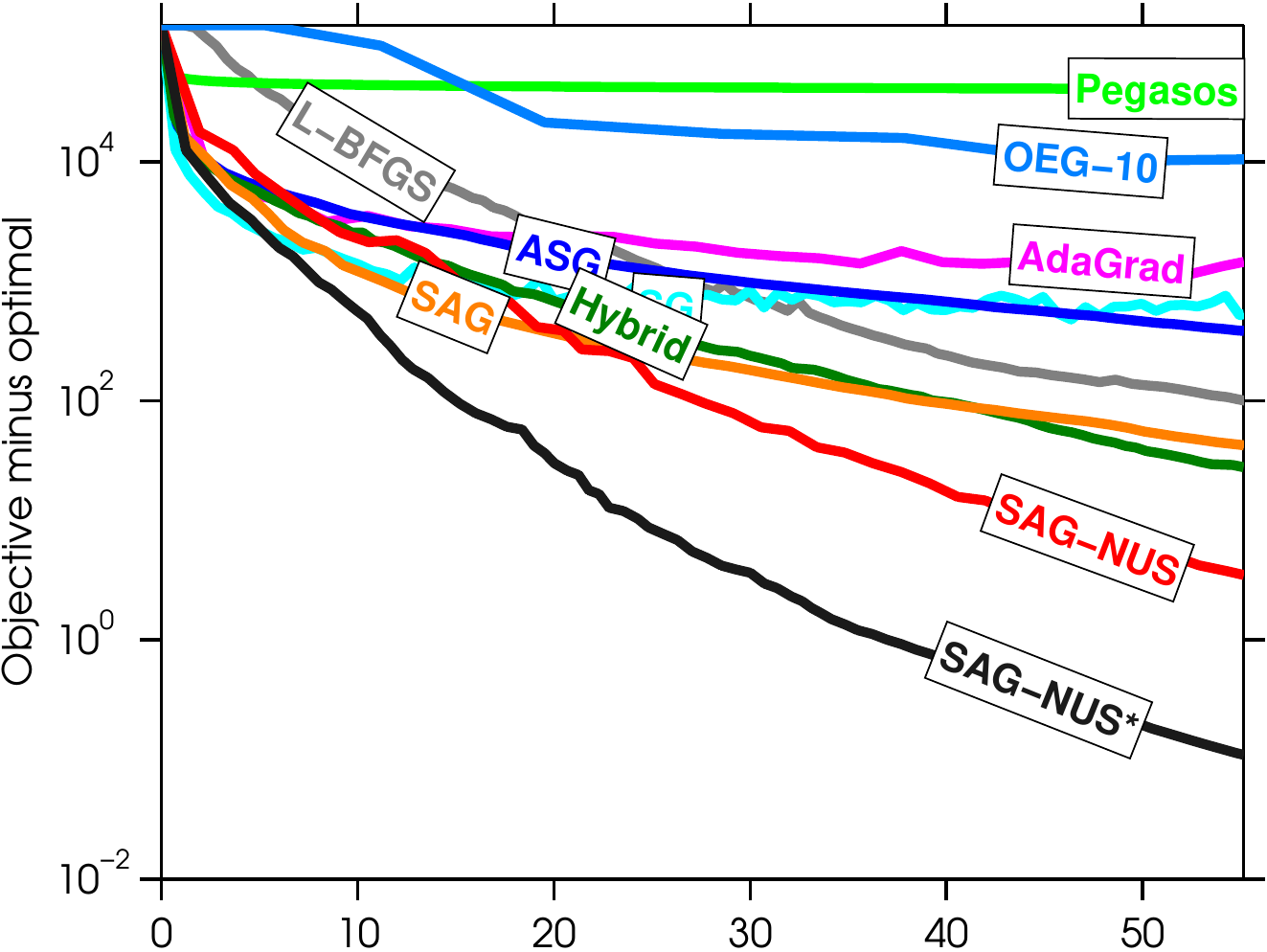}
\fig{.35}{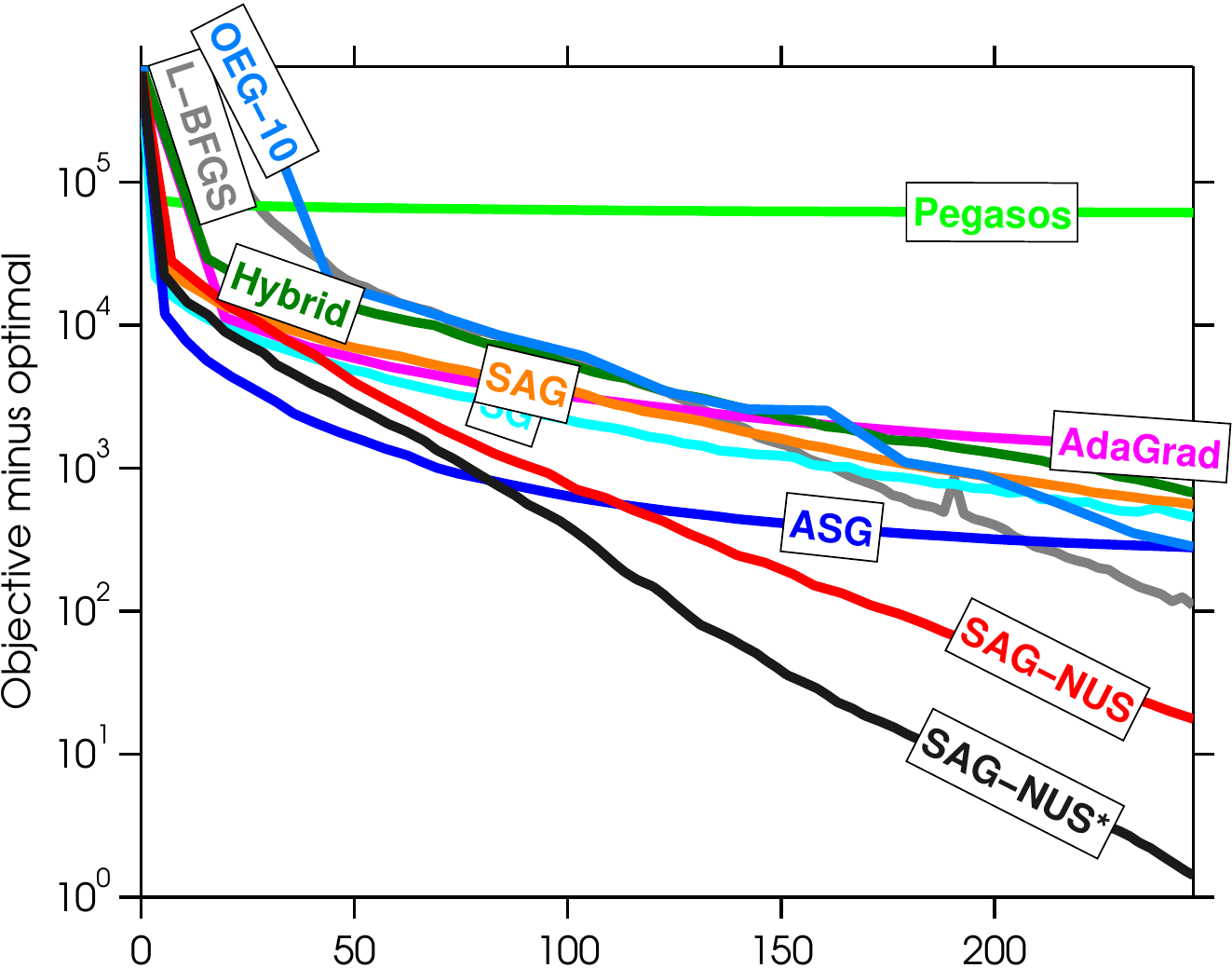}\\
\fig{.35}{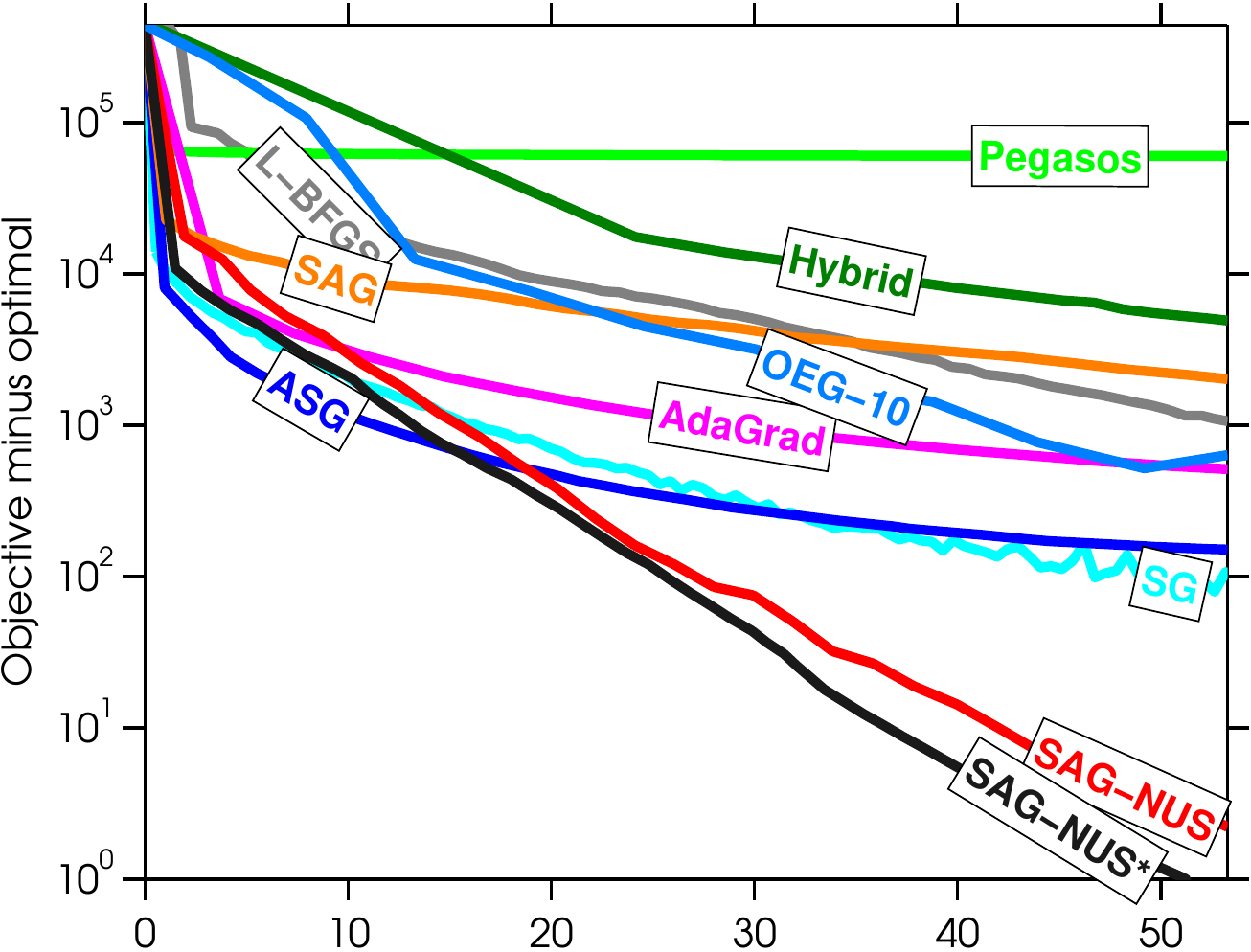}
\fig{.35}{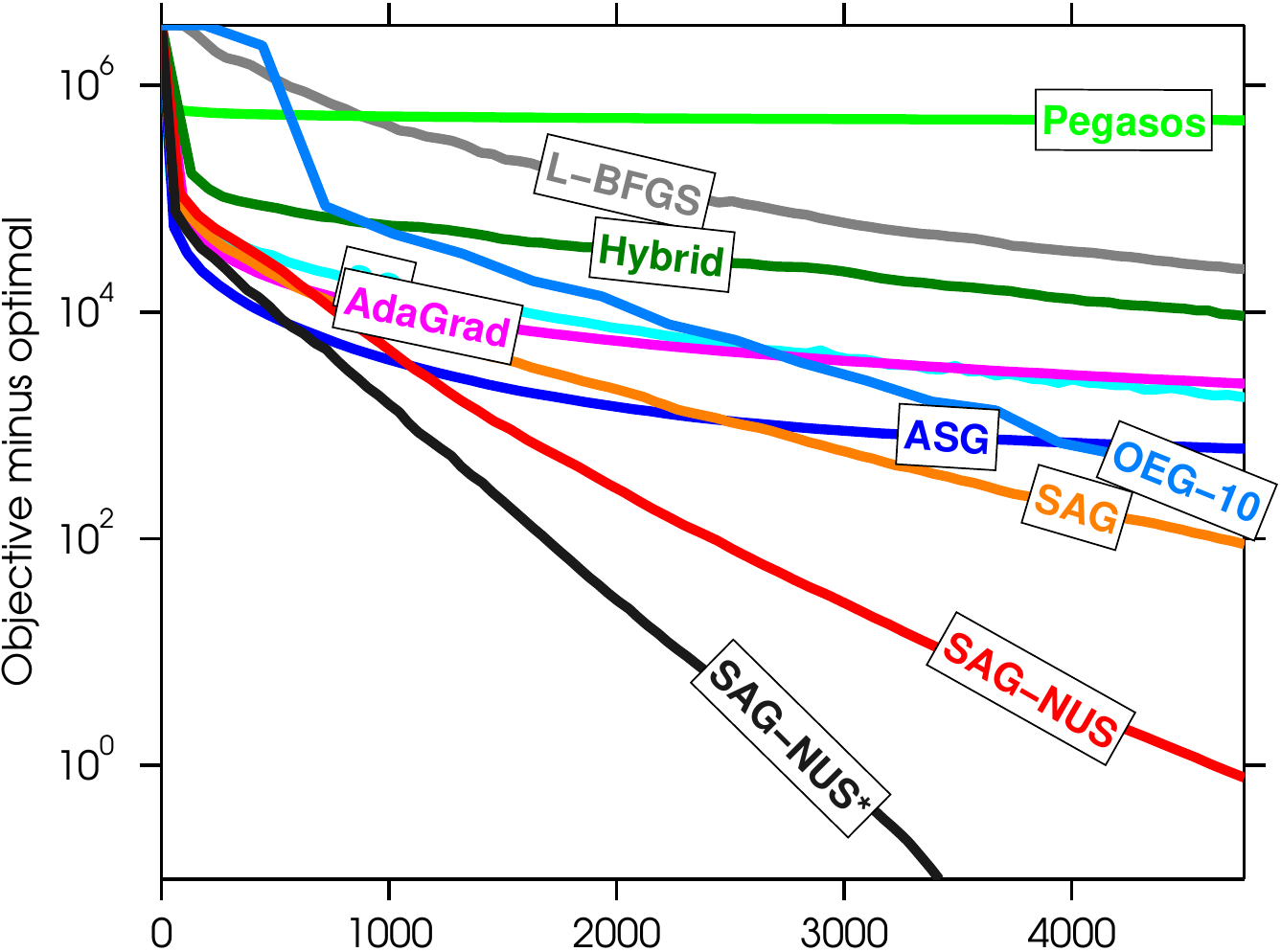}
\end{center}
\caption{Objective minus optimal objective value against time for different deterministic, stochastic, and semi-stochastic optimization strategies. Top-left: OCR, Top-right: CoNLL-2000, bottom-left: CoNLL-2002, bottom-right: POS-WSJ.}
\label{fig:trainTime}
\end{figure*}

\begin{figure*}
\begin{center}
 \fig{.35}{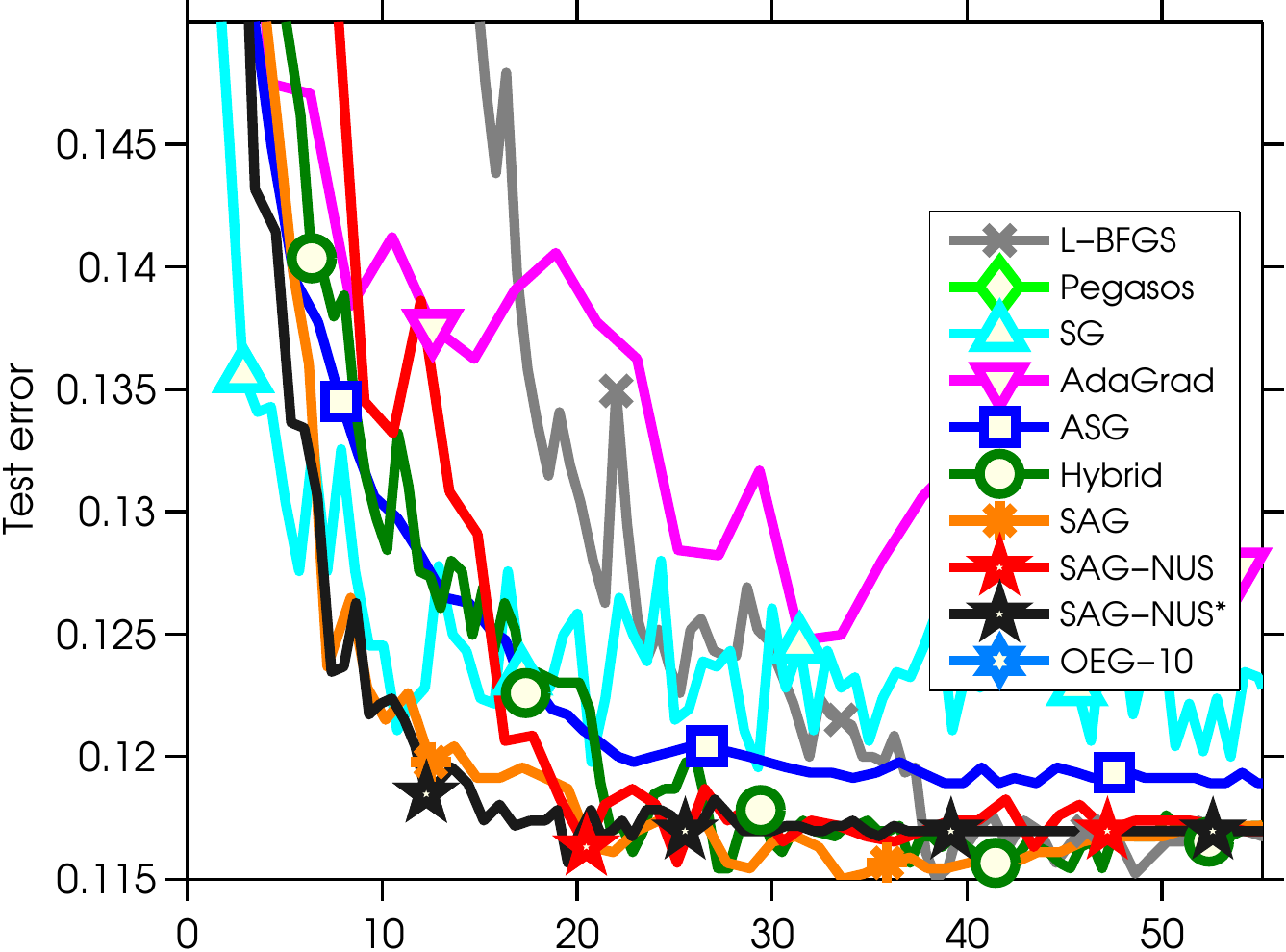}
\fig{.35}{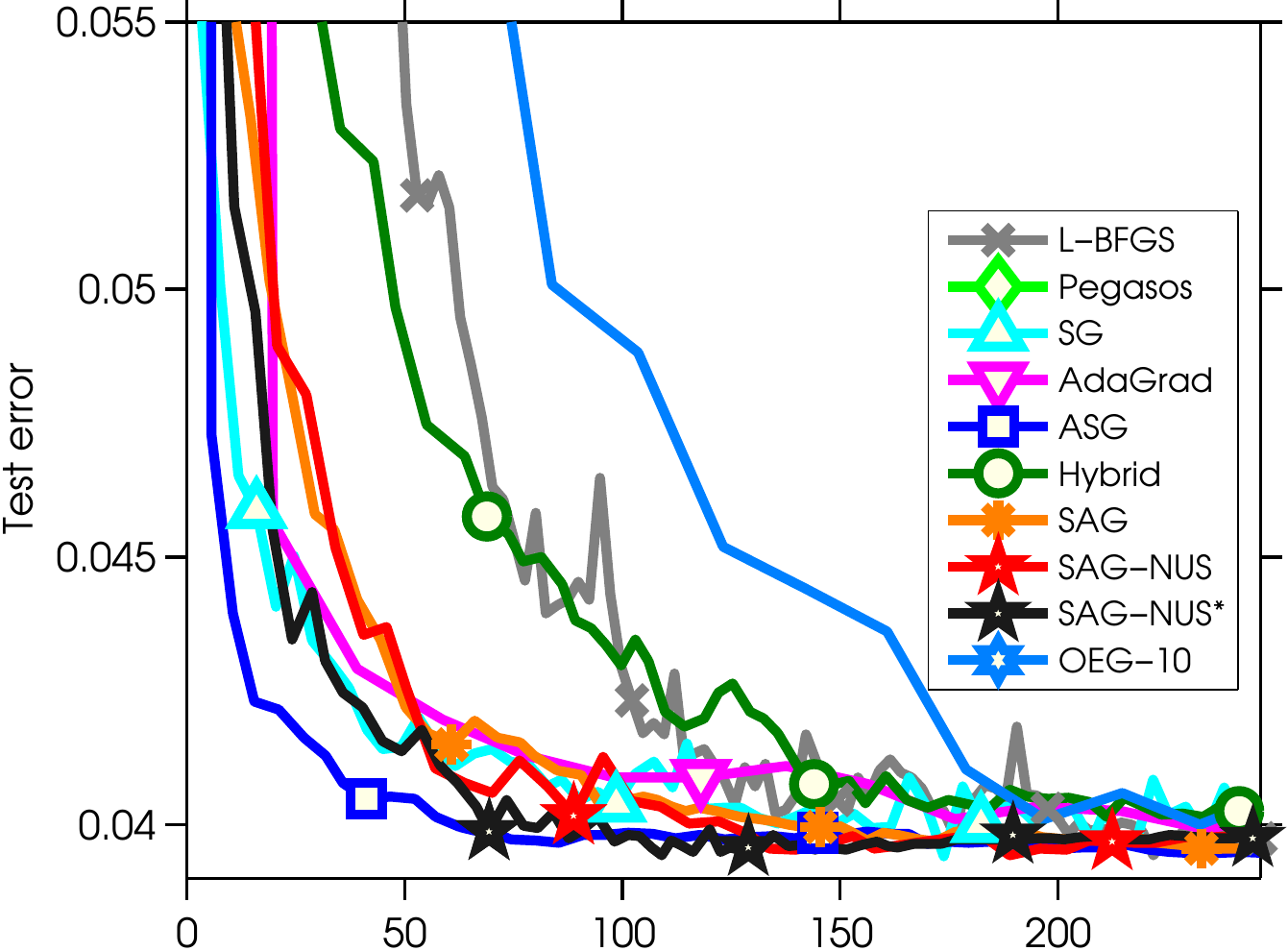}\\
\fig{.35}{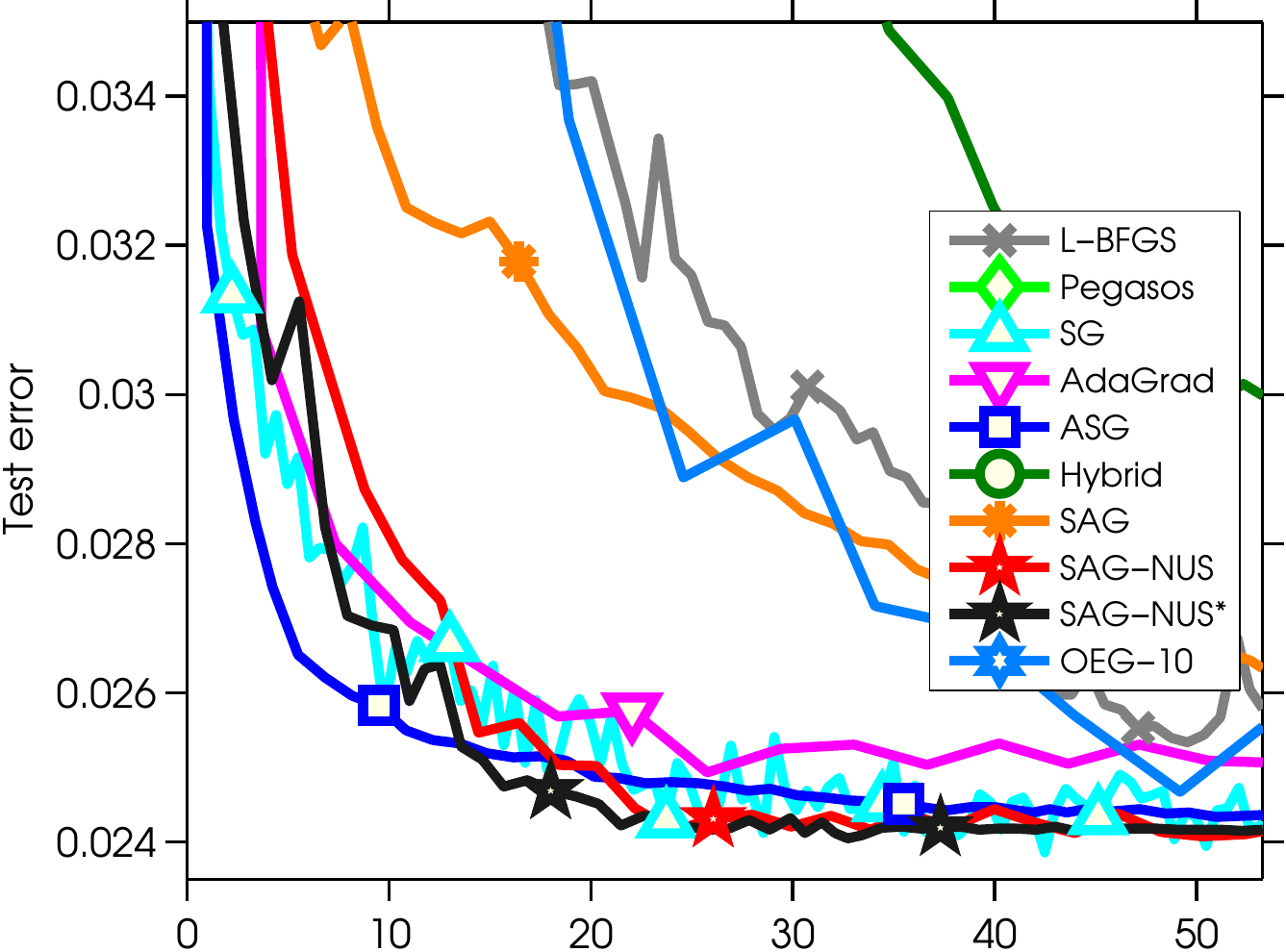}
\fig{.35}{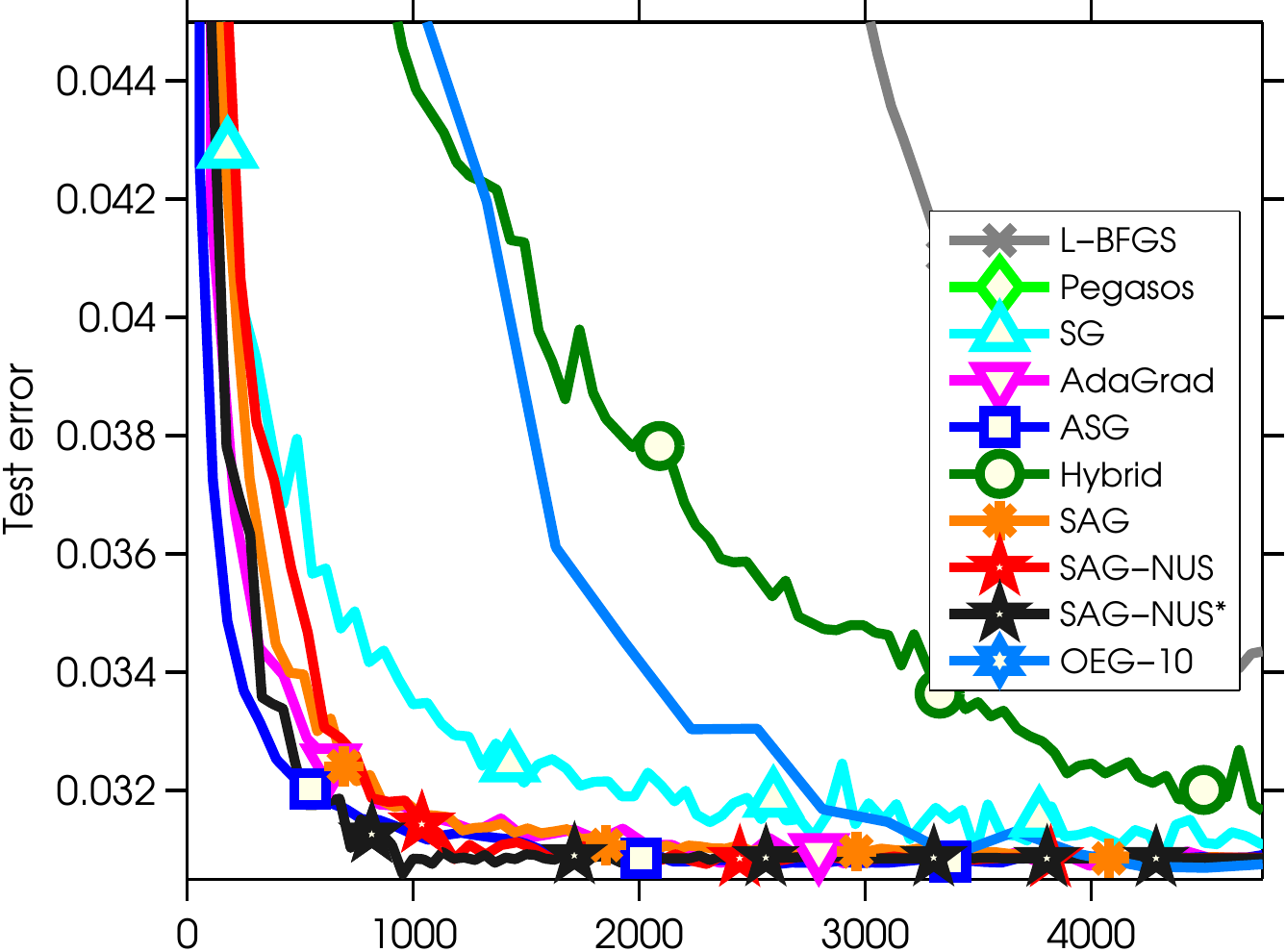}
\end{center}
\caption{Test error against time for different deterministic, stochastic, and semi-stochastic optimization strategies. Top-left: OCR, Top-right: CoNLL-2000, bottom-left: CoNLL-2002, bottom-right: POS-WSJ.}
\label{fig:testTime}
\end{figure*}

In the main body we plot the performance against the effective number of passes as an implementation-independent way of comparing the different algorithms. In all cases except SMD, we implemented a C version of the method and also compared the running times of our different implementations. This ties the results to the hardware used to perform the experiments and to our specific implementation, and thus says little about the runtime in different hardware settings or different implementations, but does show the practical performance of the methods in this particular setting. We plot the training objective against runtime in Figure~\ref{fig:trainTime} and the test error in Figure~\ref{fig:testTime}. In general, the runtime plots show the exact same trends as the plots against the effective number of passes. However, we note several small differences:
\begin{itemize}
\item \emph{AdaGrad} performs slightly worse in terms of runtime. This seems to be due to the extra square root operators needed to implement the method.
\item \emph{Hybrid} performs worse in terms of runtime, although it was still faster than the \emph{L-BFGS} method. This seems to be due to the higher relative cost of applying the L-BFGS update when the batch size is small.
\item \emph{OEG} performed much worse in terms of runtime, even with the better initialization from the previous section.
\end{itemize}
Finally, we note that these implementations are available on the first author's webpage:\\
\url{http://www.cs.ubc.ca/~schmidtm/Software/SAG4CRF.html}

\bibliography{bib}
\bibliographystyle{abbrvnat}

\end{document}